\documentclass[12pt,english]{article}
\usepackage[T1]{fontenc}
\usepackage[latin9]{inputenc}
\usepackage{geometry}
\usepackage{color}
\usepackage{babel}
\usepackage{booktabs}
\usepackage{fancybox}
\usepackage{calc}
\usepackage{mathtools}
\usepackage{url}
\usepackage{bm}
\usepackage{algorithm2e}
\usepackage{amsmath}
\usepackage{amsthm}
\usepackage{amssymb}
\usepackage{graphicx}
\usepackage{setspace}
\usepackage{wasysym}
\usepackage[authoryear]{natbib}
\usepackage[unicode=true,
 bookmarks=false,
 breaklinks=false,pdfborder={0 0 1},backref=false,colorlinks=false]
 {hyperref}
\providecommand{\tabularnewline}{\\}
\newcommand{\lyxdot}{.}
\makeatletter

\providecommand{\tabularnewline}{\\}

\theoremstyle{plain}
\newtheorem{thm}{\protect\theoremname}
\theoremstyle{definition}
\newtheorem{example}[thm]{\protect\examplename}
\theoremstyle{plain}
\newtheorem{lem}[thm]{\protect\lemmaname}
\theoremstyle{definition}
\newtheorem{defn}[thm]{\protect\definitionname}
\theoremstyle{plain}
\newtheorem{prop}[thm]{\protect\propositionname}
\theoremstyle{plain}
\newtheorem{cor}[thm]{\protect\corollaryname}
\theoremstyle{plain}
\newtheorem*{thm*}{\protect\theoremname}

\usepackage{babel}
\usepackage{adjustbox}

\def\th@newremark{\th@remark\thm@headfont{\bfseries}}

\theoremstyle{newremark}

\usepackage{bm}
\usepackage{amsfonts}
\usepackage{tikz-qtree}
\usetikzlibrary{arrows}

\providecommand{\corollaryname}{Corollary}
\providecommand{\definitionname}{Definition}
\providecommand{\lemmaname}{Lemma}
\providecommand{\propositionname}{Proposition}
\providecommand{\theoremname}{Theorem}

\usepackage{soul}%
\newcommand{\mtp}[1]{\textcolor{red}{\bf #1}}
\newcommand{\hrl}[1]{\textcolor{blue}{(\textit{#1})}}

\providecommand{\examplename}{Example}

\makeatother

\providecommand{\corollaryname}{Corollary}
\providecommand{\definitionname}{Definition}
\providecommand{\examplename}{Example}
\providecommand{\lemmaname}{Lemma}
\providecommand{\propositionname}{Proposition}
\providecommand{\theoremname}{Theorem}
\date{}
\begin{document}
\title{Sharded Bayesian Additive Regression Trees}
\author{Hengrui Luo\thanks{Lawrence Berkeley National Laboratory, Berkeley, CA, 94720, USA, E-mail:
hrluo@lbl.gov}, Matthew T. Pratola\thanks{Department of Statistics, the Ohio State University, Columbus, OH, 43210, USA, E-mail: mpratola@stat.osu.edu}}
\maketitle

\begin{abstract}
In this paper we develop the randomized Sharded
Bayesian Additive Regression Trees (SBT) model.
We introduce a randomization auxiliary variable and a sharding tree to decide partitioning of data, and fit each partition component to a sub-model using Bayesian Additive Regression Tree (BART). By observing that the optimal design of a sharding tree can determine optimal sharding for sub-models on a product space, we introduce an intersection tree structure to completely specify both the sharding and modeling using only tree structures. In addition to experiments, we also derive the theoretical optimal weights for minimizing posterior contractions and prove the worst-case complexity of SBT. 

\end{abstract}

\section{Background}

The motivation of this paper is to improve scalability of regression tree models, and in particular Bayesian
(additive) tree models, by introducing a novel model construction where we can adjust
the distribute-aggregate paradigm in a flexible way. We consider data with a continuous multivariate input $\bm{x}$,
and continuous univariate responses $y$ on the task of regression, however the underlying concepts we introduce are easily generalized to a variety of common scenarios. 

\subsection{Distributed Markov Chain Monte Carlo}

As parallel computation techniques for large datasets become popular
in statistics and data science \citep{pratola2014parallel,kontoghiorghes2005handbook,jordan_communicationefficient_2019}, the problem of distributed inference
becomes more important in solving scalability \citep{jordan_communicationefficient_2019,dobriban2021distributed}.
The basic idea of \textit{distribute-aggregate} inference can be traced
back to the bagging technique \citep{breiman1996bagging} for improving
model prediction performance. Bayesian modeling can sometimes take
model averaging into consideration \citep{raftery1997bayesian,wasserman2000bayesian}, %
by using the distribute-aggregate paradigm %
for computational scalability on big datasets.%

An important example of the distribute-aggregate paradigm arises in
Markov chain Monte Carlo (MCMC) sampling for Bayesian modeling, where
multiple independent chains could be sampled in parallel, and then
combined in some fashion. Distributed inference in MCMC can be introduced
as a \emph{transition parallel} approach, which improves the mixing
behavior for modeling with big data \citep{chowdhury2018parallel}
using appropriate detailed balance equations. %
Although transition parallel MCMC can help in general MCMC computations,
it does not address the model fitting on big data directly.%

Alternatively, to provide computational gains in the big data setting,
one is motivated to take a \emph{data parallel} approach \citep{scott_bayes_2016,wang2013parallelizing}
and execute each chain using a subset of the full dataset. \citet{scott_bayes_2016}
proposed an aggregation approach when the posterior is (a mixture
of) Gaussians; and \citet{wang2013parallelizing} devised a novel sampler
to approximate more sophisticated posteriors via Weierstrass transforms.

Our interest lies in Bayesian tree models. Unfortunately, the posterior
from Bayesian (additive) tree models belong to neither of these two
scenarios and it is unclear how to subset the data and ensure the
validity of the aggregated result for tree model posterior distributions.
Echoing the existing literature, our starting point is the consensus Monte Carlo (CMC) algorithm by \citet{scott_bayes_2016}.
CMC proceeds by splitting the full dataset $\mathcal{X}\subset\mathbb{R}^{d}$
consisting of $n$ samples, with inputs $\bm{X}=(\bm{x}_{1}^{T},\cdots,\bm{x}_{n}^{T})^{T},\bm{x}_{1},\cdots,\bm{x}_{n}\in\mathbb{R}^{d}$
and responses $\bm{y}=(y_{1},\cdots,y_{n})^{T}$, into $B$ subsets $\mathcal{X}_{i},\ i=1,\ldots,B,$ each  
with $n_{i}$ samples ($\sum_{i}^{B}n_{i}=n$), called \emph{data
shards}, such that $\mathcal{X}=\cup_{i=1}^{B}\mathcal{X}_{i}\ensuremath{}
$. The idea of CMC is to distribute each of the $B$ shards across
$B$ sub-models distributed to separate compute nodes (i.e., worker machines), and construct $B$ chains for
sub-models based on the $\mathcal{X}_{i}$ for each worker machine.
Then, one draws Monte Carlo samples for parameters, denoted by $\theta_{j}^{(i)},i=1,2,\cdots,B,j=1,2,\cdots$,
from a posterior distribution $p(\theta\mid\mathcal{X}_{i})$ at a
much cheaper computational expense as compared to the expense of drawing
samples directly from the original posterior $p(\theta\mid\mathcal{X})$.

In this approach, the aggregate posterior sample is taken to be $\theta_{j}=\sum_{i=1}^{B}w_{i}\theta_{j}^{(i)},j=1,2,,\cdots$
where $w_{i}\in[0,1]$ are (static) weights. The parameters $\theta_{j}$'s
are considered to be consensus samples drawn from the posterior $p(\theta\mid\mathcal{X})$.
These consensus draws represent the consensus belief among all the
sub-models' posterior chains, and when the sharded sub-models and
full model posteriors are both Gaussian the weights can be determined
to recover the full model's posterior. %

\subsection{Scalability of Bayesian Modeling}

Scalability issues %
introduced by either sample size, high dimensionality or intractable
likelihoods \citep{craiu2022approximate}, are especially pervasive in the
Bayesian context \citep{wilkinson2022distance}, especially in the
Bayesian inference of nonparametric models \citep{LNP2022_SAGP,zhu2022radial,katzfuss2021general,pratola2014parallel},  which usually require MCMC for model fitting. 
And while the concept of the CMC inference procedure of \citet{scott_bayes_2016}
is tempting, it is not obvious how to best select  shards, nor how to best aggregate
the resulting sharded posteriors, to effectively estimate the original
true posterior of interest for sophisticated models.

In other words, precisely how to choose shards $\mathcal{X}_{i}$,
and how to appropriately choose the weights $w_{i}$ can be a challenge.
This affects the quality of the aggregated model as well as the mixing
rates.  Compared to the CMC approach, there are at least two additional considerations to contemplate: one is how to select the design and the  number of 
shards to improve the aggregated model prediction, the other is how
to weight the sub-models for better posterior convergence and contraction.

The optimality of design can also be crucial in the predictive performance
of statistical models %
\citep{drovandi2017principles,derezinski2020bayesian,Murray:EECS-2022-258}.
Usual optimal designs (e.g., Latin hyper-cube) are difficult to elicit
\citep{derezinski2020bayesian}, therefore, optimal designs via randomized
subsampling are proposed as a computationally efficient alternative
\citep{drovandi2017principles}. In what follows, we will derive and
prove that optimal design for our models is still attainable,  conditioned on the shards.%

Posterior sampling quality is another important concern.
In one initial attempt to apply the distribute-aggregate Monte Carlo
scheme in Bayesian modeling, \citet{huggins2016coresets} focuses
on Bayesian logistic regression, and seeks to shard the dataset in
such a manner that the sharded likelihoods are ``close''
to the full-data likelihood in a multiplicative sense. They %
proposed a \emph{single} weighted aggregated data subset drawn using
multinomial sampling based on weights reflecting the sensitivity of
each observation%
, thereby selecting observations that are ``representative'' of
each cluster of data %
for each sub-model. \citeauthor{huggins2016coresets} are able to show
good performance of their methods, %
implying that ``optimal'' selection of shards and weights in CMC
could be vital, although a clear connection remained elusive.

Our primary model investigated in this paper, the Bayesian Additive
Regression Tree (BART, \citet{chipman_bart_2010}), is additionally
constrained by the computationally heavy MCMC procedure used in the
posterior sampling of BART.  Tree regression
models have the natural ability to capture and express interactions
of a high-dimensional form without too much smoothness assumptions,
while other nonparametric models need explicit terms to capture any
assumed interactions. With ensemble constructions such as BART's sum-of-trees form, tree regressions  gain improved predictions and can handle uncertainty quantification. However, as an ensemble method, a tree ensemble also suffers from computational
bottlenecks.

We propose a modified BART model that brings the idea of sharding into a cohesive Bayesian model by taking advantange of BART's tree basis functions.
Observing the two sources of computational bottleneck -- ensemble
and MCMC -- the core insight of our proposal builds on that of Bayesian
regression trees seen to date: unlike in typical regression where
the basis functions are fixed and only the parameters are uncertain,
Bayesian regression trees learn both the parameters \emph{and} the
basis functions. In other words, the tree model can serve as basis
for and approximate both response and model space. 

There are two important assumptions in our distribute-aggregate model : First, we assume that each sub-model is fit using exactly one shard. %
Second, we assume that the shards are non-intersecting, $\mathcal{X}_{i}\cap\mathcal{X}_{j}=\emptyset$,
which can be relaxed in general. %
These two assumptions simplified
our formulations and implementation but are not so restrictive as to limit the applicability
of our methods.

In this work, we take a somewhat different perspective on the challenging
problem of sharding large datasets in a useful way to perform Bayesian
inference. Specifically, while many have focused on methods of approximating
a particular posterior distribution, we instead adhere to the notion
of including all forms of uncertainty, 
including data and model
uncertainty, into our Bayesian procedure. Since all models are wrong
but some are useful \citep{box1976science}, this seems a more pragmatic
and practically useful approach.

Applying this line of thought, our proposed model   controls and
learns a sharded model space (and relevant model weights) as well
as the basis functions and parameters for each model that makes up
our collection of models. It is the unique duality of partitions and
trees that allows this to be done in a computationally effective manner.
The model space is approximated by partitions represented by a tree,
whilst the regression basis function and parameters are approximated
by functional representation of trees. The idea of learning partitions in a  data-driven manner has been explored in previous works on non-parametric modeling \citep{hrluo_2021e,LNP2022_SAGP}.

Our approach could recover some reference model of perceived interest
if the data warrants it, but practically will recover models which
most effectively shard the data thereby allowing faster computation
while also providing a better fitting model. Interestingly, our method
makes an unexpected connection to the multinomial sampling of \citet{huggins2016coresets}
via an optimal design argument, while using all of the data in a computationally
efficient way. %

\subsection{Organization }

\begin{figure}
\centering

\includegraphics[width=0.7\textwidth]{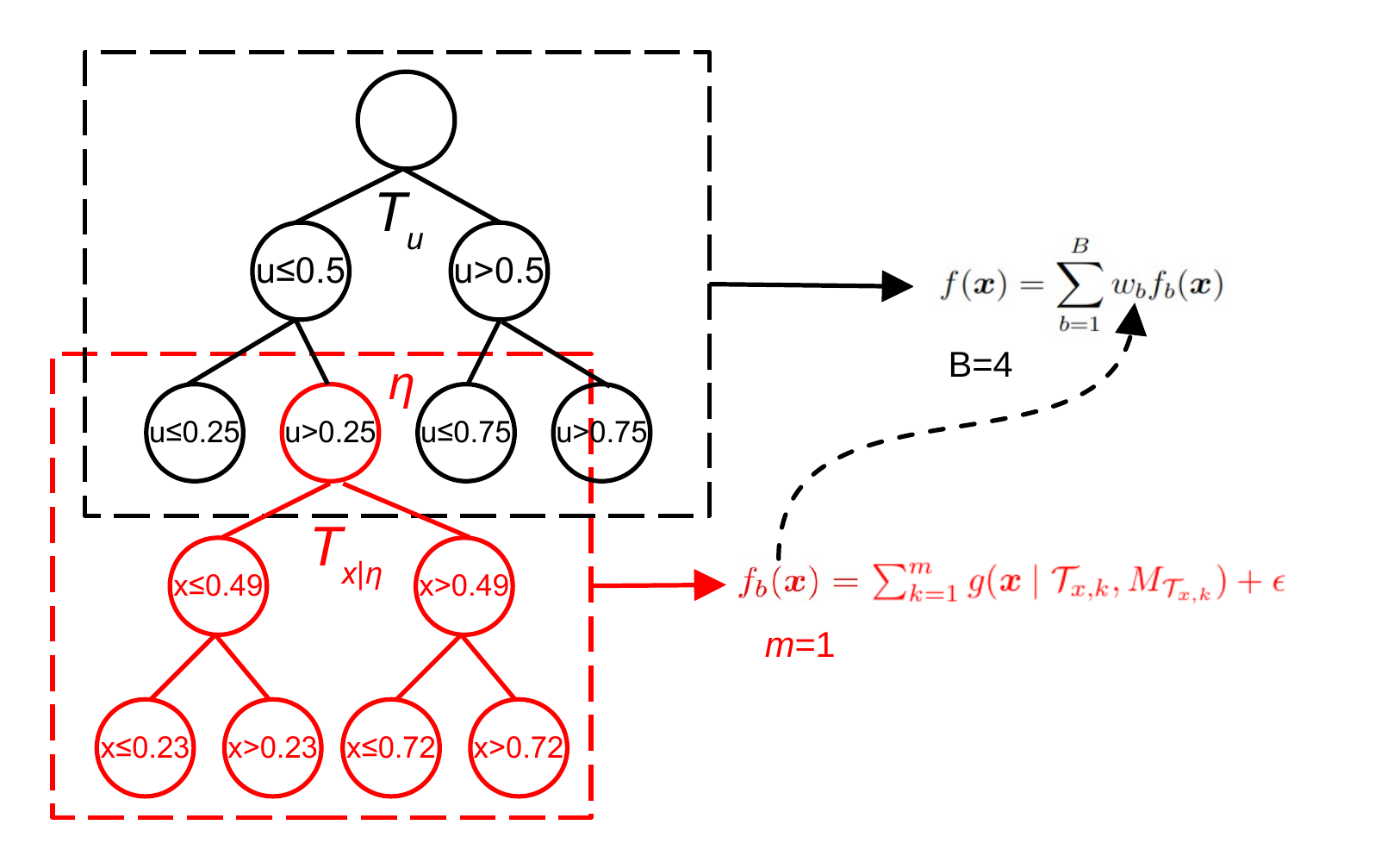} \caption{\label{fig:symbols}We graphically represent a tree (topology) $\mathcal{T}_{u}$ on the auxiliary variable $\bm{U}$ 
of depth $k=2$ and $B=4$ leaf nodes whose splitting rules are purely on variable $u$
illustrated in black, and we concatenate another tree (topology, i.e., BART with $m=1$ tree) $\mathcal{T}_{x\mid\eta_{u}}$
illustrated in red, whose root node is the leaf node $\eta=\eta_{u}$
of tree $\mathcal{T}_{u}$, but its splitting rules are purely on
input variable $\bm{X}$. We take the convention that $\mathcal{T}_{x\mid\emptyset}$
means the tree only splits on $\bm{X}$.}
\end{figure}

Our paper proceeds as follows. In Section \ref{sec:Of-Trees-and}
we review Bayesian tree models, the partitioning interpretation of trees, and develop optimal design results for single-tree models that we will
need in our sharded tree model development. In Section \ref{sec:Details-on-Model}, 
we introduce our sharded tree model by using a novel construction
involving a latent variable with an optimal sharding tree as described by algorithms; %
and then we provide some interpretation
on the model via notions of weights and marginalization, connecting
our construction back to the existing literature. %
In Section \ref{sec:Experiments-and-Applications} we explore some
examples of our proposed model and discuss complexity trade-offs. Finally, we conclude and point out future directions
in Section \ref{sec:Discussion}.

\section{Of Trees and Partitions\label{sec:Of-Trees-and}}

\subsection{Bayesian Tree Regression Models\label{sec:intro_to_BART}}

Tree models were popularized by the bagging and random forest techniques
\citep{breiman1996bagging}, and have gained attention in Bayesian forms
\citep{pratola2020heteroscedastic,chipman_bart_2010,gramacy2008bayesian}.
Tree models have been shown to be flexible in modeling complex relationships
in high-dimensional scenarios.  We briefly revisit the Bayesian tree model below, introducing
the notations we use, with an emphasis on the tree-induced partition
of the input domain.

\begin{example}
\label{exa:tree path}%
(Basis functions defined along a path) 
The general form of a tree $\mathcal{T}$ we will consider in this paper is depicted in Figure \ref{fig:symbols}, which is shown to consist of a subtree splitting on variable $u$, $\mathcal{T}_u,$ and descendant trees splitting on the input variables $\bm{X}$, shown as $\mathcal{T}_{x\vert\eta_u}.$  We will discuss $\mathcal{T}_u$ momentarily, but for now ignore this subtree and simply consider a single tree that only splits on $\bm{X}$, represented as $\mathcal{T}_{x\vert \emptyset}\equiv \mathcal{T}_x$, as shown in Figure \ref{fig:A-balanced-tree}. Each path in a tree in Figure \ref{fig:A-balanced-tree} defines a basis function splitted by the input $\bm{X}$, which is the usual setting in tree models \citep{wu2007bayesian,chipman_bart_2010}.

Given such a tree $\mathcal{T}_{u}$ 
consisting of $B$ terminal nodes, each path $\mathcal{P}_{b}$ from
terminal node $b$ back to the root node of $\mathcal{T}_u,$ say $\eta_{1},$ defines a basis
function $f_{b}$ that is formed as the product of indicator functions
resulting from the evaluation of each internal nodes rule, i.e., the
regression function corresponding to this node $b$ can be written
as a product of indicator functions where $\mathcal{I}(\cdot;S)$
means the indicator function conditioned on the indicator parameter $S$. For example, the sub-model can be an indicator function along a path $\mathcal{P}_b$ in a tree: 
\begin{align}
f_{b}({\bm{x}})=\prod_{\eta_{j}\in\mathcal{P}_{b}}\mathcal{I}\left(\bm{x}_{v_{j}};\rho_{j},c_{j}\right)=\mathcal{I}(\bm{x};R_{b})\label{eq:path_representation}
\end{align}
where $v_{j}$ is the index of splitting variable of node $\eta_{j}$,
$c_{j}$ is the corresponding splitting value and the symbol $\rho_{j}\in\{<,\geq\}$
corresponds to the appropriate inequality sign along the path $\mathcal{P}_{b}$.
As we only consider binary trees, then each interior node $\eta_{j}$
can only have a left-child or right-child. The internal structure of
$\mathcal{T}_u$ can then be fully determined by the collection $\{(v_{j},c_{j})\}$
of splitting indices $v_{j}$ and splitting values $c_{j}$ along
with these parent/child relationships.

The important observation here is that the partitioning implied by the product of indicator functions in (\ref{eq:path_representation}) induces
rectangular subregions, the $R_{b}$'s,  determined by the given structure of the tree $\mathcal{T}_x$.
Specifically, each path $\mathcal{P}_{b}$ from the root to a leaf
defines a rectangle $R_{b}$, where $\cup_{b=1}^{B}R_{b}=[0,1]^{d}$,
for which $f_{b}({\bm{x}})=1$ if ${\bm{x}}\in R_{b}$ and $0$ otherwise.
\end{example}

We model the response $y$ as $f(\bm{x})+\epsilon$ where the conditional mean
$\mathbb{E}(y\mid\bm{X}=\bm{x})=f(\bm{x})$ can be written in the following form: 
\begin{align}
f(\bm{x}) & =\sum_{b=1}^{B}w_bf_{b}(\bm{x}\mid \mu_b)=\sum_{b=1}^{B}w_bf_{b}(\bm{x})\label{eq:generic tree model}
\end{align}
where $w_b$ are the weights for single tree models (we can assume $w_b=1$ for simplicity now); $\mu_b$ is the tree terminal node parameter (as coefficients for sub-models) for terminal node $b$, the noise random variable $\epsilon\sim N(0,\sigma^{2})$ and
$f_{b}(\bm{x})$ is an indicator function defined by the rectangle
support set $R_{b}$ as defined in \eqref{eq:path_representation}.
Note that this generic tree regression model is conditioned on the tree
structure $\mathcal{T}_x$. %

We denote all the coefficients in \eqref{eq:generic tree model} as
a set $M_{\mathcal{T}_x}=\{\mu_{1},\cdots,\mu_{B}\}$. Estimating this
regression function in a Bayesian way can be achieved by putting priors
on the coefficients $\mu_{b}\in\mathbb{R}$ associated with $\mathcal{T}_x$'s leaf nodes $\eta_{x,b}$
and each of these sub-models $f_{b}$. %
Therefore, \eqref{eq:generic tree model} can simply be written as
$y(\bm{x})=g(\bm{x}\mid\mathcal{T},M_{\mathcal{T}})+\epsilon$ when
there is only one tree $\mathcal{T}$. %
For a more complex dependence scenario, \citet{chipman_bart_2010}
proposed to stack single tree models
to obtain BART, where there are $m$ tree regression functions (i.e., number of trees in a BART model) in the regression model: 
\begin{align}
y(\bm{x}) & =\sum_{j=1}^{m}g(\bm{x}\mid\mathcal{T}_{x,j},M_{\mathcal{T}_{x,j}})+\epsilon\label{eq:BART model}
\end{align}
where $\epsilon\sim N(0,\sigma^2).$ %
Usually, zero mean normal priors are assumed for the coefficients
in $M_{\mathcal{T}_{x,k}}$ along with chi-squared variance priors.
In what follows, we assume that the noise is known for simplicity,
unless otherwise is stated. With \eqref{eq:BART model}, it is not
hard to see that we are actually creating an ensemble consisting of
$B\cdot m$ simple indicator functions.

Our basic and straightforward approach of modifying the BART model \eqref{eq:BART model}
is to fit the model using 
\emph{data
shards} $\mathcal{X}_{1},\cdots,\mathcal{X}_{B}$, %
where each BART random function $g(\bm{x}\vert \mathcal{T}_{x,k},M_{\mathcal{T}_{x,k}})$ is fitted using shard $\mathcal{X}_k$ (instead of the much larger $\mathcal{X}$) to reduce the computational cost of fitting the model. %
In this way, the data shards and BART  are in one-to-one correspondence
so that we fit each shard to a BART model (i.e., $B=2^{\text{depth}}$  in generic weighted model in \eqref{eq:generic tree model}, so $m$ single trees are fitted for each of $B$ leaf nodes):
\begin{align}
y(\bm{x}) & =\sum_{k=1}^{B}\sum_{j=1}^{m}g(\bm{x}\mid\mathcal{T}_{x,j},M_{\mathcal{T}_{x,j}}, \mathcal{X}_{k})+\epsilon,\label{eq:BART model-ind2}
\end{align}
In this approach, within each BART, all $m$ trees share the same data shard. 

However, as an alternative of using the same shard across $m$ trees for one BART, we can use different shards across $m$ trees (in fact, the current codebase was most naturally extended in this way, of which (\ref{eq:BART model-ind2}) is a special case, upon which our theoretical work focuses, for simplicity of exposition).  Furthermore, different BART at different leaf nodes do not have to assume the same $m$ but different $m_1,\cdots,m_B$, and different single trees in BART can be fitted on different shards $\mathcal{X}_{k,j}$ instead of using the same  $\mathcal{X}_{k}$, namely:
\begin{align}
y(\bm{x}) & =\sum_{k=1}^{B}\sum_{j=1}^{m_k}g(\bm{x}\mid\mathcal{T}_{x,j},M_{\mathcal{T}_{x,j}}, \mathcal{X}_{k,j})+\epsilon,\label{eq:BART model-ind}
\end{align}
and the analysis of this more general form will be left as future work.

The challenge is how to do this in a way that we retain a cohesive estimation and prediction Bayesian model, facilitating both computational efficiency and straightfoward Bayesian inference.  Our approach achieves this  by
recognizing that sharding  BART %
can be induced by 
representing
the data sharding using a tree. To formalize this idea, we examine
the equivalence between a tree and a partition next.

\subsection{Equivalence between a Tree and a Partition }

Our insight begins with recognizing that a tree structure is equivalent
to a partition of a metric space (e.g., \citet{talagrand2014upper})
as shown in Example \ref{exa:tree path}, which allows us to use the
tree to approximate the model space as well as the target function
at the same time. However, binary regression trees cannot induce
 arbitrary partitions. In what follows, we focus on the rectangular
partitions that \emph{can} be induced by regression trees.

The particular partitioning of a dataset could be determined by the
construction of a tree structure, %
in particular its depth, number of terminal nodes, and the selection
of splitting points. 
For balanced trees, we note that any
balanced $2^{k}$ partitioning can be represented as a balanced binary
tree of depth $k$. %
Homogeneous partitions can be induced by balanced trees while heterogeneous
partitions can be induced by unbalanced trees, as in Figure \ref{fig:A-balanced-tree}.

\begin{figure}
\centering
\includegraphics[width=0.8\textwidth]{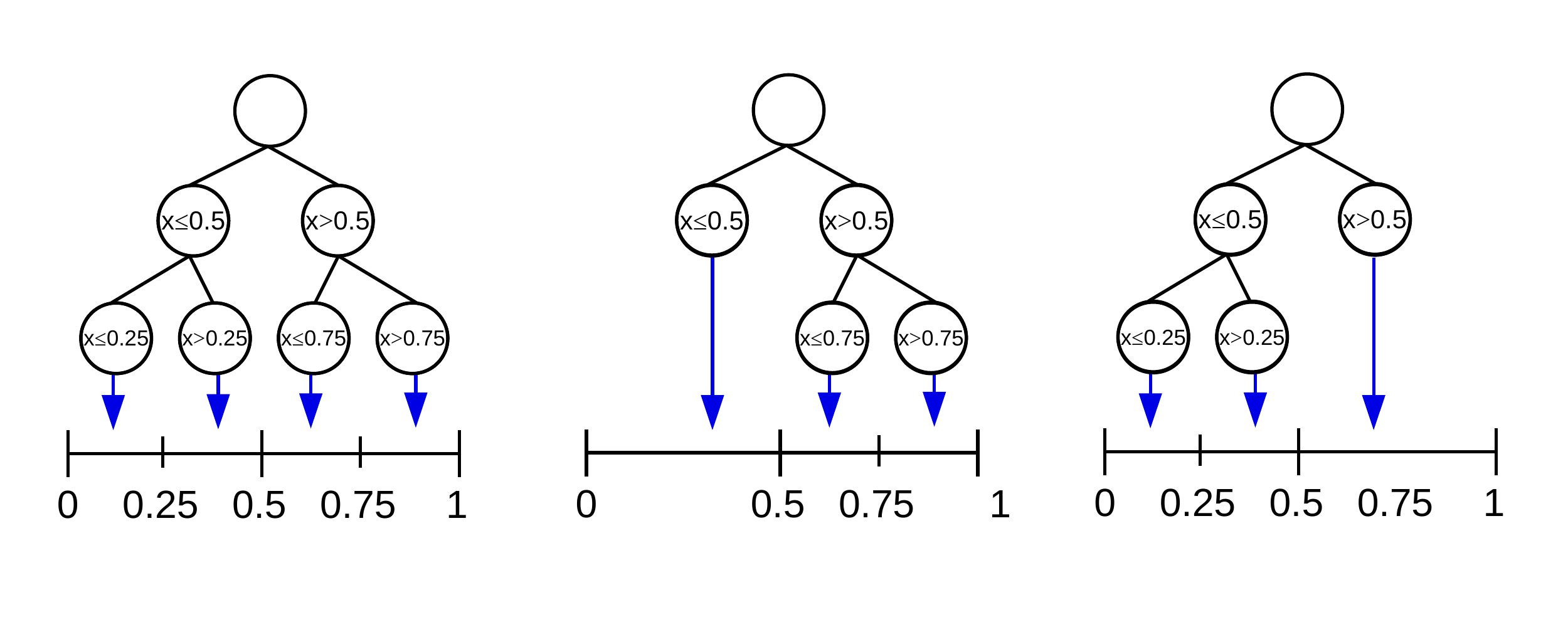} \caption{\label{fig:A-balanced-tree}A balanced tree of depth $k=2$ and its
pruned sub-trees %
which induced partition on $[0,1]$ (left), and similarly unbalanced trees (middle) and (right). In each of the interior node,
we present the splitting rule at $\eta_{j}$ using $\bm{x}_{i,v_{j}}[\rho_{j}]c_{j}$.
For simplicity, when the inequality of an interior node holds, we
proceed through the left-child node, otherwise we proceed through
the right-child node. %
}
\end{figure}

\subsection{Representing Sharding as a Tree Partition}
Now we are ready to take advantage of the partition induced by a
tree and use such a partition in creating the data shards in \eqref{eq:BART model-ind}. 
In order to improve
the scalability of model fitting through sharding, %
we will outline a model-based construction to sharding motivated by 
an optimal design criteria. In this way, we develop a novel tree model
that incorporates both modeling and sharding for the input variable
$\bm{x}$ in one shot, and we propose to learn the sharding with another
tree structure involving an auxilliary random variable $u$. 

The sharding created in this way will allow us to attain the distribute-aggregate
scheme suggested in \citet{scott_bayes_2016} %
in one unifying model where both the degree of sharding and the sharding weights {\em are informed by the data}.

If we take \eqref{eq:BART model-ind} as the regression model, then
the sharding on the input $\bm{X}$ cannot be separated from the
estimating procedure of $\mathbb{E}\left(y\mid\bm{X}\right)$. To decouple the
sharding from the estimation without breaking the model, we observe
that the shard partitions can be generated based on an independent
random variable %
$\bm{U}$ while the regression basis functions
can be generated based on observations $\bm{X}$. This allows more
flexibility in the creation of the partitions, which allocates the
observations into different shards, while the regression is still
exclusively based on $\bm{X}$. This seemingly straightforward step
essentially allows us to present the following regression model with $B$ (the same number as the number of leaves in $\mathcal{T}_u$) BART models:
\begin{align}
y(\bm{x},\bm{u}) & =\sum_{k=1}^{B}g_{\text{BART}}(\bm{x}\mid\mathcal{T}_{x\mid\eta,k},M_{\mathcal{T}_{x\mid\eta,k}},\mathcal{T}_{u})%
+\epsilon.\label{eq:BART model-ind-uvar}
\end{align}

More precisely, let us introduce the auxiliary vector $\bm{U}=(U_{1},\ldots,U_{n})^{T}$
where each $U_{i}\sim_{iid}\text{Unif}(0,1)$ (or another distribution
as illustrated in Appendix \ref{distribution example}),
which is used to augment the observations $\bm{X}$ of size $n$
and define the augmented $n\times(p+1)$ matrix of inputs as $\bm{V}=[\bm{X};\bm{U}].$
We introduce $\bm{u}$ as a device to induce the data shards
via a \emph{partitioning tree} $\mathcal{T}_{u}$ that only splits
on this auxiliary variable $\bm{U}$, as depicted in Figure \ref{fig:symbols}. From expression \eqref{eq:BART model-ind-uvar},
$\mathcal{T}_u$ effectively partitions the dataset  $\bm{X}$ as well, since
the augmented input $\bm{V}=[\bm{X};\bm{U}]$ binds the observed values
of $\bm{X}$ and $U$ together. Because the $\bm{X}$ and auxiliary variable $\bm{U}$ are binded, the sample counts can be done on either $\bm{U}$ or $\bm{X}$. As we will see later,
the effect of $\bm{U}$ on the posterior of interest can be removed simply
by marginalizing it out.

Details of fitting and predicting from our proposed model described
by \eqref{eq:BART model-ind-uvar} are formalized in Algorithms \ref{alg:RST-Algorithm-Fitting}
and \ref{alg:RST-Algorithm-Predict}, and its graphical representation
is postponed to Figure \ref{fig:Example-of-product-partition}. For
now, we show how the model can be constructed using a tree $\mathcal{T}_{u}$
dedicated to creating data shards, as shown in Figure \ref{fig:Balanced-tree-Tu-for-partition} in Appendix \ref{distribution example}.

In the rest of the paper, we will focus on the setting where $U\sim\text{Unif}(0,1)$,
but point out that there are many possibilities here with different pros and cons, depending on the goal.
We focus on one specific sharding tree $\mathcal{T}_{u}$ motivated by an optimal design argument, 
and show that such a sharding  improves
posterior concentration rates in the aggregated model.

\subsection{\label{sec:OD_in_reg_tree}Optimal Designs in a Tree }

After establishing the basic notations and notions used in our tree models,
it is natural to consider what an optimal design \citep{fedorov_theory_2013}
means for this specific tree model $\mathcal{T}_{u}$ for creating
data shards. We will start by considering the optimal design in the
fixed location setting (i.e., $\bm{U}$ is deterministic and fixed),
then we generalize our discussion to random location setting (i.e.,
$\bm{U}$ follows a probability distribution $\mathbb{P}$). Considering  all terminal node trees $\mathcal{T}_{x\mid \eta_u}$ to have only root node for the moment, we focus on the design of $\mathcal{T}_u$. 
This setting is conditionally equivalent to the following generic linear regression model for the data with inputs
$\bm{u}\in\mathbb{R}^{d}$ and response $y\in\mathbb{R}^{1}$ as follows:
\begin{align}
y(\bm{u}) & =F(\bm{u})^{T}\cdot\bm{\beta}+\epsilon,\label{eqn:regmodel}
\end{align}
where the normal noise $\epsilon\sim N(0,\sigma^{2})$ and suppose
the variance parameter $\sigma^{2}>0$ is fixed and known.
In the context of regression
model \eqref{eqn:regmodel}, we use the following vector notations $\bm{u}_{i}  =\left(\bm{u}_{i,1},\cdots,\bm{u}_{i,d}\right)^T\in\mathbb{R}^{d\times 1},$ and $F({U}_{i}) =\left(f_{1}(\bm{u}_{i}),\cdots,f_{B}(\bm{u}_{i})\right)^T\in\mathbb{R}^{d\times 1},\bm{\beta}\in\mathbb{R}^{B\times1}.$ The $\bm{beta}$ is the coefficient for the single tree $\mathcal{T}_u$ in the optimal design discussion in this section. Then we can write the relevant matrices as 
\begin{align}
\bm{U} & =\left(\begin{array}{c}
{U}_{1}\\
\vdots\\
{U}_{n}
\end{array}\right)=\left(\begin{array}{ccc}
\bm{u}_{1,1} & \cdots & \bm{u}_{1,d}\\
\vdots &  & \vdots\\
\bm{u}_{n,1} & \cdots & \bm{u}_{n,d}
\end{array}\right)\in\mathbb{R}^{n\times d},\\
\bm{F}(\bm{U}) & =\left(\begin{array}{c}
F(U_{1})\\
\vdots\\
F(U_{n})
\end{array}\right)=\left(\begin{array}{ccc}
f_{1}(\bm{u}_{1}) & \cdots & f_{B}(\bm{u}_{1})\\
\vdots &  & \vdots\\
f_{1}(\bm{u}_{n}) & \cdots & f_{B}(\bm{u}_{n})
\end{array}\right)\in\mathbb{R}^{n\times B},
\end{align}

If we use a partition induced by a tree $\mathcal{T}_{u}$, then the
number of partitions $B$ also defines $B$ basis indicator function
corresponding to each partition component. A popular design-of-experiments
criterion, known as \emph{D-optimality} %
\citep{shewry_maximum_1987,chaloner1995bayesian,fedorov_theory_2013},
has the optimality criterion function defined as $\phi(\bm{U})=\texttt{det}\left(\sigma^{-2}\bm{F}(\bm{U})^{T}\bm{F}(\bm{U})\right)$
for a linear model, %
and $\bm{F}(\bm{U})$ corresponds to the design matrix formed by
evaluating $B$ regression functions $f_{j}(\bm{U}),j=1,\ldots,B$
at the designed input settings $\bm{U}$.

Therefore, the tree model defined in \eqref{eq:generic tree model}
can be formulated in the form of \eqref{eqn:regmodel} by using indicator
functions as the basis function in the design matrix $F(\bm{U})$, using indicator functions of leaves. 
The matrix $\bm{F}(\bm{U})^{T}\bm{F}(\bm{U})$ reduces to 

{\scriptsize{}{}{} 
\[
\bm{F}(\bm{U})^{T}\bm{F}(\bm{U})=\left(\begin{array}{cc}
\sum_{k=1}^{n}\bm{1}(u_{k}\in\eta_{1})\bm{1}(u_{k}\in\eta_{1}) & \sum_{k=1}^{n}\bm{1}(u_{k}\in\eta_{1})\bm{1}(u_{k}\in\eta_{2})\\
\sum_{k=1}^{n}\bm{1}(u_{k}\in\eta_{2})\bm{1}(u_{k}\in\eta_{1}) & \sum_{k=1}^{n}\bm{1}(u_{k}\in\eta_{2})\bm{1}(u_{k}\in\eta_{2})\\
\vdots\\
\sum_{k=1}^{n}\bm{1}(u_{k}\in\eta_{5})\bm{1}(u_{k}\in\eta_{1}) & \sum_{k=1}^{n}\bm{1}(u_{k}\in\eta_{5})\bm{1}(u_{k}\in\eta_{2})
\end{array}\right).
\]
}%

In the current setup, the partitions are  determined by the leaves of tree
$\mathcal{T}_{u}$ and are non-overlapping. Therefore, the $\sum_{k=1}^{n}\bm{1}(u_{k}\in\eta_{i})\bm{1}(u_{k}\in\eta_{j})=\sum_{k=1}^{n}\delta_{ij}$
will have a range in  $[0,n]$. If we introduce overlapping partitions,
then the upper bound of this range becomes $n\times B$ since each
observation can be counted up to $B$ times in each of the partition
component. This is a major difference between non-overlapping and
overlapping partitions (a.k.a. overlapping data shards) in terms of
experimental design, since the optimal criteria becomes more vague
if we can arbitrarily re-count the samples (i.e., allow partitions
to share observations).

As we noted above, we want to modify the classical optimality criteria
for tree models. To start with, we examine how the classical optimality
criteria like D-optimality behaves for the tree regression model.
If we consider the optimal designs for the tree structure shown in
Figure \ref{fig:A-balanced-tree} with $n=4$ samples, for any $n_{j}>1$
we must have some $n_{j^{\prime}}=0,j^{\prime}\neq j,$ for which
$\phi$ would evaluate to $0.$

A regression tree can be expressed in such a linear formulation with
the functional representation explained in \eqref{eq:path_representation}.
Besides the viewpoint of taking indicator function of leaf components,
we can take an alternative formulation tracing along the paths in
a tree. This viewpoint allows us to understand the multi-resolution
nature provided by the tree regression.

\subsubsection{Tree Optimal Design for Fixed Locations}

First, we resume our discussion on D-optimality and consider the situation
where the tree structure $\mathcal{T}$ is fixed and we want to choose
the optimal design for $\bm{U}$ deterministically.

For an $n$-sample design in the inputs $\bm{u}_{1},\ldots,\bm{u}_{n}$
with $n_{1}$ inputs mapping to terminal node region $\bm{R}_{1}$,
$n_{2}$ inputs mapping to terminal node $\bm{R}_{2}$, etc., the criterion simplifies to %
: 
\begin{align}
\phi_{n}(\bm{u}_{1},\ldots,\bm{u}_{n}\mid\mathcal{T})\coloneqq\prod_{b=1}^{B}\frac{n_{b}}{n}\propto\texttt{det}{\tiny\left(\begin{matrix}n_{1} & 0 & 0 & 0 & \cdots & 0\\
0 & n_{2} & 0 & 0 & \cdots & 0\\
0 & 0 & n_{3} & 0 & \cdots & 0\\
0 & 0 & 0 & n_{4} & \cdots & 0\\
\vdots & \vdots & \vdots & \vdots & \ddots & 0\\
0 & 0 & 0 & 0 & 0 & n_{B}
\end{matrix}\right)}\geq0,\label{eq:phi_x_T}
\end{align}

Assume that $\sum_{b=1}^{B}n_{b}=n$ and $n/B\in\mathbb{Z}$, then
by arithmetic-geometric mean inequality or Lagrange multiplier, the
product $\phi$ is bounded from above by $\left(\frac{n}{B}\right)^{B}$,
which is attained by letting $n_{1}=\cdots=n_{B}=\frac{n}{B}$. This
means that we want to put the same number of observed sample points
in each node partition of $\mathcal{T}$. Generally, the following
lemma is self-evident via Euclidean division. This result is different
from the usual regression setting where the locations can be in $\mathbb{R}$
\citep{fedorov_theory_2013}, as here the entries can only assume
values in $\mathbb{N}$. In other words, when conditioned on the tree
structure $\mathcal{T}=\mathcal{T}_{u}$, a D-optimal design assigns
``as equal as possible'' number of sample points to each leaf node. 
\begin{lem}
\label{lem:maximize_CTOD} Assume that the number of sample $n=q\cdot B+r,r<B\leq n,q\in\mathbb{N}$
where $B$ is the number of leaf nodes in a fixed tree $\mathcal{T}$.
Then, the optimal criterion function $\phi_{n}(\bm{u}_{1},\ldots,\bm{u}_{n}\mid\mathcal{T})$
defined in \eqref{eq:phi_x_T} is maximized by assigning $q$ samples
to any of $(B-r)$ leaf nodes; and $(q+1)$ samples to the rest $r$
nodes. The maximum is $\phi_{n}^{\max}(\mathcal{T})\coloneqq\left(\frac{q}{n}\right)^{B-r}\cdot\left(\frac{q+1}{n}\right)^{r}$,
where the maximizer is unique up to a permutation of the leaf indices. 
\end{lem}

This lemma points out
that: in the case $n\geq B$, the design criterion is maximized by
a balanced design, i.e. placing inputs as fairly as possible in each
of the terminal nodes. This latter case $n\geq B$ is what we consider
in the rest of the paper.

Note that if we consider random assignment of $n$ samples to $B$
different leaf nodes, the chance of getting $\phi_{n}^{\max}$ can
be modeled by a multinomial model. As the number of leaves $B$ increases, the chance that a random assignment attains optimal design tends to zero (see Figure \ref{fig:Lemma_pics} in Appendix \ref{sec:Optimal-Design-for}).
Therefore, a random sharding strategy would have a very slim chance
of getting any optimal design (that reaches $\phi_{n}^{\max}$), especially
for a large $n$. See Appendix \ref{sec:Optimal-Design-for} for discussions
on trees with constrained leaves and Appendix \ref{sec:Equivalence-Theorem}
for other optimality criteria.

\subsubsection{Tree Optimal Design for Random Locations}

Second, we consider the situation where the tree structure $\mathcal{T}$
can be fixed, but the observed locations are random $\bm{u}_{1},\ldots,\bm{u}_{n}\overset{\text{i.i.d.}}{\sim}\mathbb{P}$.

To generalize from \eqref{eq:phi_x_T}, we use $\bm{U}$ to denote
these random locations and study the expected optimality  for random
locations, where $\phi_{n}(\mathbb{P}\mid\mathcal{T})\coloneqq\mathbb{E}\phi(\bm{u}_{1},\ldots,\bm{u}_{n}\mid\mathcal{T})$
needs to be maximized. Now it is not possible to optimize the observed axiliary 
locations $\bm{U}$ since they are random; instead, we want to consider
what kind of tree structure $\mathcal{T}$ may offer an optimal design
on average. Therefore, we take expectation of \eqref{eq:phi_x_T}
with respect to $\mathbb{P}$, where random variables $n_{b}$ are
  the number of samples in the leaf node $\eta_{b}$ corresponding
to rectangle $R_{b}$: 
\begin{align}
& %
\mathbb{E}\phi(\bm{u}_{1},\ldots,\bm{u}_{n}\mid\mathcal{T}) =\mathbb{E}\prod_{b=1}^{B}\frac{n_{b}}{n}\text{ (by definition }n_{b}\text{'s are r.v.s.)}  =\frac{n!}{n^{B}}\prod_{b=1}^{B}\mathbb{P}\left(\bm{u}\in R_{b}\right)\label{eq:UTOD_derivation}
\end{align}

Equation \eqref{eq:UTOD_derivation}  generalizes the \emph{D-optimality}
criterion \eqref{eq:phi_x_T} when we can only determine the distribution,
instead of exact locations of $\bm{U}$. Conditioned on the tree structure
$\mathcal{T}$ we want to remove the dependence of optimality criterion
on the sample size $n$. Note that the $B$ partitions induced by
$R_{1},\cdots,R_{B}$ are determined, and the probability of getting
$n_{j}$ samples in node $\eta_{j}$ for $j=1,2,\cdots,B$~%
is $\frac{n!}{n_{1}!\cdots n_{B}!}\prod_{b=1}^{B}\mathbb{P}\left(\bm{u}\in R_{b}\right)^{n_{b}}$,
where $\bm{u}$ is one sample from $\mathbb{P}$. The random variables
$n_{1},\cdots,n_{B}$ are not independent, but their joint distribution
is multinomial. Note that the moment generating function of this multinomial
is $M(t_{1},\cdots,t_{B})=\biggl(\sum_{b=1}^{B}\mathbb{P}\left(\bm{u}\in R_{b}\right)\cdot\exp(t_{b})\biggr)^{n}$
\citep{kolchin_random_1978}, then the product expectation $\mathbb{E}\prod_{b=1}^{B}\frac{n_{b}}{n}=\left(\frac{1}{n}\right)^{B}\left(\mathbb{E}\prod_{b=1}^{B}n_{b}\right)$
can be computed as $\left(\frac{1}{n}\right)^{B}\frac{\partial}{\partial t_{1}\cdots\partial t_{B}}M(0,\cdots,0)=\left(\frac{1}{n}\right)^{B}\prod_{b=1}^{B}\mathbb{P}\left(\bm{u}\in R_{b}\right)\cdot n!$.
As an optimality criterion, we want to remove the effect of sample
size so that designs of different sizes are comparable. 
Dropping the dependence on $n$, our optimality criterion can be described
as follows.

\begin{defn}
($B$-expected optimal tree) Under the assumption where the observed
locations are drawn from $\mathbb{P}$ and let us fix the number of
partition components $B$ first, and consider the following optimality
criterion (dropping the multiplier depending on the sample size $n$),
which is a functional of the tree structure $\mathcal{T}$: 
\begin{align}
\phi(\mathcal{T}\mid\mathbb{P}) & \coloneqq\prod_{b=1}^{B}\mathbb{P}\left(\bm{u}\in R_{b}\right).\label{eq:phi_T}
\end{align}
We call the tree $\mathcal{T}_{B}^{\max}$ that maximizes \eqref{eq:phi_T}
the \textit{$B$-expected optimal tree} with respect to $\mathbb{P}$. 
\end{defn}
We take the convention that the product is taken over all non-zero
probabilities. The convention is needed for continuous input. If the
cut-points are chosen from a discrete set (as is typically the case in Bayesian tree models), then the
probability of getting a null partition set is zero. 

To develop some intuition of finding the optimal partition under \eqref{eq:phi_T},
let us first forget about the tree structure $\mathcal{T}$ and only
choose a partition consisting of $R_{b},b=1,\cdots,B$. It follows
from arithmetic-geometric-mean inequality that: 
\begin{prop}
\label{prop:general optimal}For the number of partition components
$B$ fixed, the optimality criterion \eqref{eq:phi_T} is maximized
by taking $\mathbb{P}\left(\bm{u}\in R_{1}\right)=\cdots=\mathbb{P}\left(\bm{u}\in R_{B}\right)$,
but there may not exist any $\mathcal{T}_{B}^{\max}$ that attains
this maximum value. \label{prop:general_opt_phi_T} 
\end{prop}

While the conclusion of Proposition \ref{prop:general_opt_phi_T}
remains true, the partition described later in Corollary \ref{coro:Finv_partition}
cannot always be realized by a tree $\mathcal{T}$. %
But any partition induced by a tree $\mathcal{T}$
must have rectangles as partition components (or boundaries of partition
components must be parallel to coordinate axes). Therefore, \textit{not
all} $B$-expected optimal partitions can be induced by a tree. In the case of a uniform measure, we have the following corollary.

\begin{cor}
\label{coro:uniform_measure}When $\mathbb{P}$ is a uniform measure
on a compact set (e.g., $[0,1]^{d}$) on $\mathbb{R}^{d}$, the criterion
\eqref{eq:phi_T} can be written as $\phi(\mathcal{T}\mid\mathbb{P})=\prod_{b=1}^{B}Vol\left(R_{b}\right)$. 
\end{cor}

\begin{proof}
It follows directly from \eqref{eq:phi_T} and $\mathbb{P}\left(\bm{u}\in R_{b}\right)=Vol(R_{b})$. 
\end{proof}

Obviously, the existence of $\mathcal{T}_{B}^{\max}$ depends on both
the geometry of domain and the number of partitions $B$. However,
a special case of particular interest in application \citep{balog2015mondrian}
is when $\mathbb{P}$ is a uniform measure on the unit hypercube. Below
is our first main result, stating that on the unit hypercube $[0,1]^{d}$,
$B$-expected optimal partitions are always realizable through some
tree. 

\begin{cor}
\label{cor:uniform optimal}When $\mathbb{P}$ is a uniform measure
on $[0,1]^{d}\subset\mathbb{R}^{d}$, $\mathcal{T}_{B}^{\max}$ exists. 
\end{cor}

\begin{proof}
By Corollary \ref{coro:uniform_measure}, it suffices to partition
the $[0,1]^{d}\subset\mathbb{R}^{d}$ into $B$ rectangles with equal
volumes, which is ensured by the Proposition 1' and 2' in \citet{kong1987decomposition}. 
\end{proof}
The following corollary below follows from taking grid sub-division
(for each of $d$ dimensions) and the definition of multivariate c.d.f.
$\mathbb{F}(\alpha_{1},\cdots,\alpha_{d})\coloneqq\mathbb{P}(u_{1}\leq\alpha_{1},\cdots,u_{d}\leq\alpha_{d})$.
Although it works for a wider class of measures, it comes with a stricter
restriction on $B$. 
\begin{cor}
\label{cor:cdf optimal}When $\mathbb{P}$ is a probability measure
where its cumulative distribution function $\mathbb{F}$ exists and
is invertible (with probability one) with inverse function $\mathbb{F}^{-1}$
on $\mathbb{R}^{d}$, and the number of partition components is $B=B_{0}^{d}$
for some positive integer $B_{0}$, $\mathcal{T}_{B}^{\max}$ exists.
\label{coro:Finv_partition} 
\end{cor}

\begin{proof}
The criterion \eqref{eq:phi_T} is maximized by taking the partition
components $R_{b}=S_{1}^{b}\times\cdots\times S_{d}^{b},b=1,\cdots,B=B_{0}^{d}$,
where 
\[
S_{i}^{b}\in\{(a_{j}^{i},a_{j+1}^{i}]\mid\mathbb{F}(1,\cdots,1,a_{j+1}^{i},1,\cdots,1)-\mathbb{F}(1,\cdots,1,a_{j}^{i},1,\cdots,1)=1/B_{0},j=1,\cdots,B_{0}\},
\]
and $i=1,\cdots,d$. That is, the partition components $S_{i}$ are
selected from the marginal partition on the $i$-th dimension $[a_{1}^{i},a_{2}^{i}),[a_{2}^{i},a_{3}^{i}),\cdots,[a_{B_{0}}^{i},a_{B_{0}+1}^{i}]$
such that marginally each interval would have $1/B_{0}$ probability
mass w.r.t. the marginal measure in the $i$-th dimension. Therefore,
each $R_{b}$ would have probability mass $1/B=1/B_{0}^{d}$ w.r.t.
the joint measure $\mathbb{P}$. 
\end{proof}
By Corollary \ref{coro:uniform_measure}, if the observed locations
are randomly selected from a uniform distribution on $[0,1]^{d}$,
then the optimal sharding consisting of $B=B_{0}^{d}$ components
is just regular grid partition with equal number of samples assigned
to each component as in Lemma \ref{lem:maximize_CTOD}. We only discuss
finitely many leaves, i.e., $B<\infty$ in this paper, but briefly
discuss the design when $B\rightarrow\infty$ in Appendix \ref{sec:Asymptotic-Optimal-Designs}. 

Conditioned on the tree structure, an expected B-optimal design assigns
``as equal probability mass as possible'' to each leaf node, where
the probability mass is computed using the probability measure for
splitting variable $u$. To sum up our findings so far, we define
an optimality criterion that is more suitable for tree models, and
derive results on what an optimal design looks like in these scenarios.
In addition, we also discuss whether an optimal design, if it exists,
can be attained by a partition induced by a tree. 
When the distribution is uniform, an evenly (in terms of volume) sized
sharding is optimal from regression design point of view (Corollary
\ref{cor:uniform optimal}). 
When the distribution is not uniform but is absolutely continuous with invertible cumulative
distribution functions, an optimal design always exists and can
be attained by a tree (Corollary \ref{cor:cdf optimal}); otherwise, the optimal design may
exist, but  it may not be attained by a splitting tree $\mathcal{T}_{u}$. %

An optimal design for $\mathcal{T}_{u}$ that maximizes \eqref{eq:phi_T}
splits the observations according to the summary as above and coincide with the suggestions by \citet{huggins2016coresets}. We do not
have to control the locations of $\bm{X}$ in each shard, since it
can be automatically determined once the structure of $\mathcal{T}_{u}$
is determined (hence the partition on the $\bm{U}$ domain).  In essence, we convert the problem of constructing a sharding into
the problem of fitting a tree $\mathcal{T}_{u}$.

\subsection{Product Partition and Intersection Trees}

Now, we have justified that the size of sharding follows the ``probability
inverted'' rule in Corollary \ref{cor:uniform optimal} and \ref{cor:cdf optimal}. %
To automatically create the sharding in practice, we introduce
our notion of intersection tree below. The idea is to introduce the
auxilliary variable $\bm{U}$, which acts as a device that allows us to
utilize the partition generated by $\mathcal{T}_{u}$ to create the
sharding. %
The optimal design for $\bm{U}$ will create the associated partition on $\bm{U}$, therefore also on $\bm{X}$. Uniformly distributed $\bm{U}$ has equal-volume optimal design on $\bm{U}$ domain, by changing the distribution of $\bm{U}$ (or correlate with $\bm{X}$) we can attain different shardings on $\bm{X}$ domain. 

The advantage of binding the latent variable $\bm{U}$ with $\bm{X}$ is that the model can learn the best partition of $\bm{X}$ in terms of the topology of $\mathcal{T}_u$. This reduces the problem of sharding to choosing a best apriori distribution for $\bm{U}$.  Then, during the model fitting, the optimal sharding can also be learned.%

In this section, we formalize the idea of ``using auxiliary $\bm{U}$
to form a partition scheme on $\bm{X}$ domain'' as illustrated by Figure \ref{fig:Example-of-product-partition}.
Following our convention, we know that $\mathcal{T}_{x\mid\emptyset}$
means there is not any $\mathcal{T}_{u}$ and we are studying the
single tree for $\bm{X}$. 
Considering a single-tree regression model
based on $\mathcal{T}_{x}=\mathcal{T}_{x\mid\emptyset}$ for observations
$\bm{y}$ and inputs $\bm{X}$, the model (\ref{eqn:regmodel})
becomes
\begin{align}
y(\bm{x})=g(\bm{x}\mid\mathcal{T}_{x},\bm{\beta})+\epsilon
\end{align}
where $g(\bm{x};\mathcal{T}_{x},\bm{\beta})=F(\bm{x})^{T}\bm{\beta}$
is the approximating function, in the form of linear combinations
of indicators induced by the particular tree structure $\mathcal{T}_{x}=\mathcal{T}_{x\mid\emptyset}$
splitting only on inputs $\bm{X}$ with corresponding coefficients
$\bm{\beta}$ as described earlier.%
To embed our sharding notion within the modeling framework, we instead
consider the augmented model 
\[
y(\bm{x}\mid \bm{u})=g(\bm{x}\mid \bm{u};h(\mathcal{T}_{x\mid\eta_{u}},\mathcal{T}_{u}),\bm{\beta})+\epsilon
\]
where the augmented model involves the sharding tree $\mathcal{T}_{u}.$
There are many forms one might consider for combining the sharding
tree $\mathcal{T}_{u}$ with the regression tree $\mathcal{T}_{x}=\mathcal{T}_{x\mid\eta_{u}}$
For instance, $h(\mathcal{T}_{x\mid\eta_{u}},\mathcal{T}_{u}):=\mathcal{T}_{x}$
reduces to $\mathcal{T}_{x\mid\emptyset}$ and implies no sharding
at all.%
\begin{figure}
\centering

\includegraphics[width=0.8\textwidth]{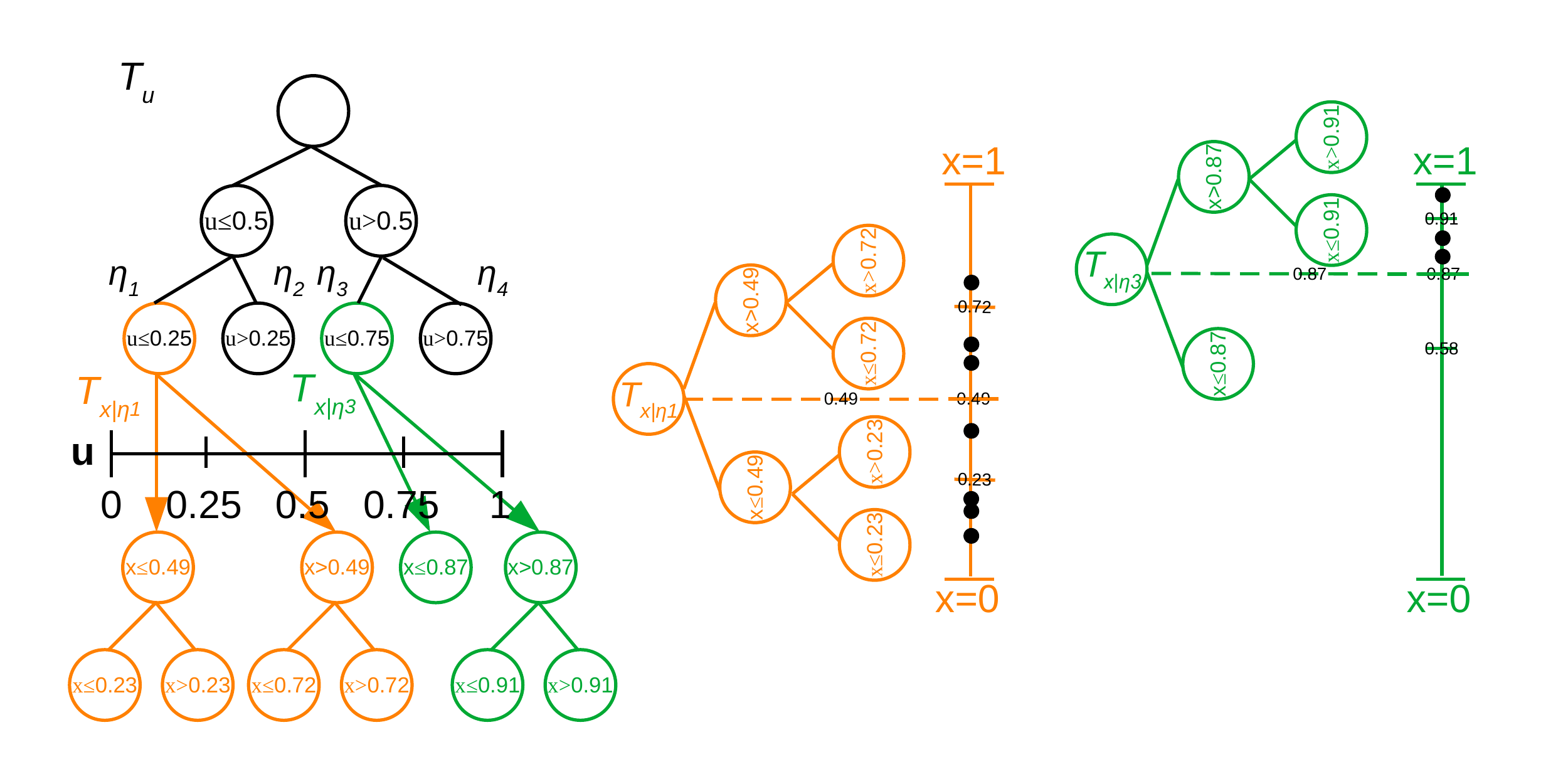}\caption{\label{fig:Example-of-product-partition}An example of intersection
tree based on the partition. We show two BART models represented by $m=1$ tree.}
\end{figure}

Let $\mathcal{T}_{u}$ be a tree splitting on the latent variable
$u$ inducing a partitioning of $[0,1]$. Then we use $\mathcal{T}_{x\mid\eta_{u}}$
to represent the tree splitting on the input variable $\bm{X}$ on the
shard partition %
corresponding to the node $\eta_{u}$. The graphical representation
is in Figure \ref{fig:Example-of-product-partition}. Our sharded
tree model is thus defined using $\mathcal{T}$ as 
\begin{equation}
y(\bm{x}\mid \bm{u})=g(\bm{x},\bm{u};\mathcal{T},\bm{\beta})+\epsilon\label{eq:single-tree}
\end{equation}
and the posterior is, 
\[
\pi(\mathcal{T},\bm{\beta},\sigma^{2}\vert\bm{y},\bm{X},U)\propto L(\bm{y};\bm{X},\bm{U},\mathcal{T},\bm{\beta},\sigma^{2})\pi(\mathcal{T})\pi(\bm{\beta})\pi(\sigma^{2})
\]
where $\pi(\bm{\beta}),\pi(\sigma^{2})$ follow their usual choices
and $\pi(\mathcal{T})=\pi(h(\mathcal{T}_{x},\mathcal{T}_{u}))=\pi(\mathcal{T}_{x}\cap\mathcal{T}_{u})=\pi(\mathcal{T}_{x\vert\eta_{u}})\pi(\mathcal{T}_{u})=\pi(\mathcal{T}_{x})\pi(\mathcal{T}_{u})$,
indicating that the sharding is independent of $\mathcal{T}_{x}$
a priori, and furthermore, taking $\pi(\mathcal{T}_{x\vert\eta_{u}})=\pi(\mathcal{T}_{x})$
informs independence of the prior on $\mathcal{T}_{x},$ i.e. the
usual tree prior for $\mathcal{T}_{x}$ would be used.  In particular, the typical depth prior for the regression tree component of the model is with respect to the root of the $\mathcal{T}_{x\vert \eta_u},$ not the root of the intersecting tree $\mathcal{T}.$

The posterior of interest is then recovered by marginalizing out $\bm{U}$,
therefore, we are averaging over different possible shardings induced
by $\mathcal{T}_{u}$: 
\[
\pi(\mathcal{T}_{x},\bm{\beta},\sigma^{2}\vert\bm{y},\bm{X})=\int_{\mathcal{T}_{u}}\int_{\bm{U}}\pi(\mathcal{T},\bm{\beta},\sigma^{2}\vert\bm{y},\bm{X},\bm{U})\pi(\mathcal{T}_{u})\pi(\bm{U})d\bm{U}d\mathcal{T}_{u}.
\]
As foreshadowed, the choice of prior on $\mathcal{T}_{u}$ and the
way how $\bm{U}$ is generated plays the key role in how the sharding affects
the model. 
When $\mathcal{T}_{x\mid\eta}$ are all containing nothing more than a root node, then the intersected tree model admits a design that is completely determined by the sharding tree $\mathcal{T}_u$. In this situation, we want to optimize $B$-expected optimality criteria for the sharding tree (e.g.,\eqref{eq:constrained_optimal_design}).

However, instead of carrying out a ``frequentist direct optimization'' to yield an optimal design for the tree, we take a Bayesian approach to sample from all possible designs for $\mathcal{T}_u$. When the $\mathcal{T}_{x\mid\eta}$ is no longer a root tree but with their own BART sum-of-tree  structures, the problem of maximizing an optimality criteria of the whole intersected tree $\mathcal{T}_u\cap\mathcal{T}_{x\mid\eta}$ becomes intractable due to the size of search space. The Bayesian intuition in intersection tree is that if we conditioned on the structure  $\mathcal{T}_{x\mid\eta}$'s, the optimal design of $\mathcal{T}_u$ will allow the posterior sampling from each $\mathcal{T}_{x\mid\eta}$ to converge much faster to the posterior mode. Nonetheless, a lot of insights could be achieved by considering the sharding induced by optimal design of $\mathcal{T}_u$, which coincides with  the suggestion in \citet{huggins2016coresets}.

\section{Sharded Bayesian Additive Tree \label{sec:Details-on-Model}}

In this section, we extend our construction of intersection tree model
to what we call \emph{sharded Bayesian additive regression trees} (SBT). %
The SBT model can be summarized as: Induce sharding using a single tree $\mathcal{T}_u$, and fit each shard using a BART $\mathcal{T}_{x\mid\eta}$ and aggregate.
\subsection{Fitting and Prediction Algorithms}
\label{sec:fitting and predicting}

From the generic form of \eqref{eq:generic tree model} and the specification
in \eqref{eq:single-tree}, we can write down our ensemble model,
\begin{equation}
y(\bm{x},\bm{u})   =\sum_{k=1}^{B}\sum_{j=1}^{m}g(\bm{x}; \bm{u}\mid\mathcal{T}_{x,j},M_{\mathcal{T}_{x,j}}, \mathcal{X}_{k}, \mathcal{T}_u)+\epsilon,\label{eq:BART model-ind2-copy}
\end{equation}
Assuming $m=1$ for simplicity and marginalizing out the variable $u$,  yield the formal expression: 
\begin{equation}
y(\bm{x})=\sum_{b=1}^{B}w_{b}\cdot f_{b}(\bm{x};\mathcal{T}_{x\mid \eta _b},\bm{\beta})+\epsilon,\label{eq:ensemble-tree}
\end{equation}
where we consider a set of weights $w_{1},\cdots,w_{B}$ from marginalization, whose
sum $\sum_{b=1}^{B}w_{b}=1$ for a fixed number of shards $B$. %
We use $f_{b}(\bm{x};\mathcal{T}_{x\mid \eta _b},\bm{\beta})=\mathcal{T}_{x\mid\eta}$  (i.e., BART models for each $\eta_u$ leaf node) to emphasize that the additive components are from the collection
$\mathcal{E}$, as defined by conditionally independent (but combined
additively) BART models. The weights can be understood as ``relative
contribution'' of each shard sub-model, and it is natural to ask
the question of how to choose each of these $w_{b}$'s to attain
an ``optimal overall model''. 

With our intersection tree topology in Figures \ref{fig:symbols}, conditioning on each sharding induces a set of weights in our ensemble. When the distribution of auxilliary
variable $\bm{U}$ and the tree structure of $\mathcal{T}_u$ are  well-chosen, we can show that the resulting sharding tree model  achieves the B-expected optimal design
we defined above. In other words, the changing $\mathcal{T}_{u}$ (as $\mathcal{T}_u$ updates during MCMC iterations)
serves as a collection of latent designs for weighting the
BART models $f_{b}$ in the ensemble under consideration. To clarify the idea, we
will introduce our fitting and prediction algorithms first, then provide
different explanations of this procedure.
\begin{algorithm}[ht!]
\rule[0.5ex]{1\columnwidth}{1pt}
\begin{itemize}
\item \textbf{Input.} The full data set consisting of both $\bm{X},\bm{y}$.
Tree priors for $\mathcal{T}_{u}$ and $\mathcal{T}_{x\mid\eta}$'s.
The number of MCMC iterations $N_{MCMC}$ (See Appendix \ref{sec:API_SBT} for APIs). 
\item \textbf{Output.} Samples from the posteriors of $\mathcal{T}_{u}$
and $\mathcal{T}_{x\mid\eta}$'s. 
\end{itemize}
\begin{enumerate}
\item Generate auxiliary variables $\bm{U}$. Sample $\bm{U}\sim\text{Unif(0,1)}$
 i.i.d. and make the auxilliary dataset $\bm{V}=[\bm{X},\bm{U}]$. %
\item For $k$ in $1:N_{MCMC}$ do 
\begin{enumerate}
\item Let $\mathcal{T}_{u}^{\text{old}}=\mathcal{T}_{u}$. 
\item Propose a regression tree $\mathcal{T}_{u}^{new}$ using some proposal.%
\item Partition $[0,1]$ using the partition representation of $\mathcal{T}_{u}$
by terminal nodes $\eta_{u}^{1},\cdots,\eta_{u}^{B}$ of $\mathcal{T}_{u}$,
therefore, partition the data into $\bm{X}=\cup_{j=1}^{B}\bm{X}_{j}$.\label{enu:Partition--using} 
\item Compute the acceptance ratio $\tau=\max(1,\tau^{*})$.\label{enu:Fit-accept-Ratio}%
We use Metropolis-Hasting step $\text{Unif}(0,1)<\tau$ to decide whether
accept: 

\item If accept, let $\mathcal{T}_{u}=\mathcal{T}_{u}^{\text{new}}$. %
If reject, let $\mathcal{T}_{u}=\mathcal{T}_{u}^{\text{old}}$.%

\item For each terminal node $\eta_{u}^{j}$ of $\mathcal{T}_{u}$ do\label{enu:Inner-loop} 
\begin{enumerate}
\item Let $\mathcal{T^{\text{old}}}_{x\mid\eta_{u}^{j}}=\mathcal{T}_{x\mid\eta_{u}^{j}}$. 
\item Propose a regression tree $\mathcal{T}_{x\mid\eta_{u}^{j}}$ with
$\bm{X}_{j}$ as input; $\bm{y}_{j}$ as output (with the tree priors
for $\mathcal{T}_{x\mid\eta_{u}^{j}}$) including splitting point.
Denote this $\mathcal{T}_{x\mid\eta_{u}^{j}}$ as $\mathcal{T}_{x\mid\eta_{u}^{j}}^{\text{new}}$.\label{enu:Fit-.-Propose} 
\item Compute the acceptance ratio\label{enu:Fit-accept-Ratio-1} {\footnotesize{}{}{}$$\tau_{j}=\max\left(1,\frac{\ell\left(\left.\mathcal{T}_{x\mid\eta_{u}^{j}}^{\text{new}}\right|\bm{X}_{j},\bm{y}_{j},\bm{u},\mathcal{T}_u\right)}{\ell\left(\left.\mathcal{T}_{x\mid\eta_{u}^{j}}^{\text{old}}\right|\bm{X}_{j},\bm{y}_{j},\bm{u},\mathcal{T}_u\right)}\right),$$}
and use Metropolis-Hasting step $Unif(0,1)<\tau_{j}$ to decide whether
accept.
\item If accept, let $\mathcal{T}_{x\mid\eta_{u}^{j}}=\mathcal{T}_{x\mid\eta_{u}^{j}}^{\text{new}}$
and update the likelihoods. 
If reject, let $\mathcal{T}_{x\mid\eta_{u}^{j}}=\mathcal{T}_{x\mid\eta_{u}^{j}}^{\text{old}}$. 
\end{enumerate}
\item Update the BART sub-model parameters as usual, and append the $\mathcal{T}_{u}$ and $\mathcal{T}_{x\mid\eta_{u}^{j}}$'s
to the output array. 
\end{enumerate}
\end{enumerate}
\rule[0.5ex]{1\columnwidth}{1pt}

\caption{\label{alg:RST-Algorithm-Fitting}Sharded Bayesian regression trees
 fitting algorithm.}
\end{algorithm}

Algorithm \ref{alg:RST-Algorithm-Fitting} allows us to draw samples
from the posterior of the $\mathcal{T}_{u}$ and $\mathcal{T}_{x\mid\eta_{u}^{j}}$,
conditioned on the data only. %
In Algorithm \ref{alg:RST-Algorithm-Fitting}
step \ref{enu:Fit-accept-Ratio}, the proposal distribution factors $p\left(\left.\mathcal{T}_{u}^{\text{new}}\right|\bm{u}\right)$
and $p\left(\left.\mathcal{T}_{u}^{\text{old}}\right|\bm{u}\right)$
represent how likely the sharding tree structure $\mathcal{T}_{u}$
can be proposed given the auxilliary variable
$u$. 
The sharding tree structure $\mathcal{T}_{u}$ partitions
the input $\bm{X}$ into shards $\bm{X}_{j},j=1,\cdots B$; and simultaneously
partitions responses $\bm{y}$ into shards $\bm{y}_{j},j=1,\cdots,B$.
This step immediately reveals the fact that we
cannot perform fitting of $\mathcal{T}_{u}$ and $\mathcal{T}_{x\mid\eta_{u}^{j}}$
separately, since a newly proposed $\mathcal{T}_{u}$ immediately
determines the $\ell\left(\left.\mathcal{T}_{x\mid\eta_{u}^{j}}\right|\bm{X}_{j},\bm{y}_{j}\right)$
in the same expression. 

Note that we do not update the whole intersection tree structure $\mathcal{T}_{u}\cap\mathcal{T}_{x}$
jointly, but update $\mathcal{T}_{u}$ and $\mathcal{T}_{x\mid\eta_{u}}$'s
sequentially. Theoretically, we can equivalently perform a joint update
using the following accept-reject ratio $\max(1,\tau^{*})$ where
\[
\tau^{*}= \frac{\left[\prod_{j=1}^{B}\ell\left(\left.\mathcal{T}_{x\mid\eta^{j}_u}\right|\bm{X}_{j}^{\text{new}},\bm{y}_{j}^{\text{new}}\right)\right]\cdot p\left(\left.\mathcal{T}_{u}^{\text{new}}\right|\bm{u}\right)}{\left[\prod_{j=1}^{B}\ell\left(\left.\mathcal{T}_{x\mid\eta^{j}_u}\right|\bm{X}_{j}^{\text{old}},\bm{y}_{j}^{\text{old}}\right)\right]\cdot p \left(\left.\mathcal{T}_{u}^{\text{old}}\right|\bm{u}\right)}.
\]
This requires us to propose not only $\mathcal{T}_{u}^{\text{new}}$
but also all $\mathcal{T}_{x\mid\eta_{u}^{j}}^{\text{new}},j=1,\cdots,B$
jointly at the same time since we want to obtain a new intersection tree structure. 
This formulation would eliminate the need of additional inner loop
in step \ref{enu:Inner-loop} and the computation of accept-reject
ratio in step \ref{enu:Fit-accept-Ratio-1}, compared to the current
Algorithm \ref{alg:RST-Algorithm-Fitting}. However, this joint proposal
would practically result in slow computation and mixing in MCMC due
to its high dimensional nature. Also, such a joint proposal is not suitable for parallel
computation. Instead, we utilize the Metropolis-Hasting step as shown below. 

In Algorithm \ref{alg:RST-Algorithm-Fitting} step \ref{enu:Fit-.-Propose},
we note that such a tree $\mathcal{T}_{u}$ is still partitioning
the whole domain $\mathcal{X}$, although its structure is completely
dependent on $\bm{X}_{j}$. After fitting the SBT model, we want to
use the posterior sample sequence of $\mathcal{T}_{u}$ and $\mathcal{T}_{x\mid\eta_{u}}$'s.
The prediction algorithm is slightly different since, as shown below, the
auxiliary variable $u_{*}$ needs to be drawn as well. This Algorithm
\eqref{alg:RST-Algorithm-Predict} gives one sample of $g(x_{*}\mid\mathcal{T}_{x\mid\eta_{x_{*}}},\mathcal{T}_{u},u_{*})$.

\noindent 
\begin{algorithm}
\rule[0.5ex]{1\columnwidth}{1pt}
\begin{itemize}
\item \textbf{Input.} Samples from the posteriors of $\mathcal{T}_{u}$
and $\mathcal{T}_{x\mid\eta}$'s. Predictive location $x_{*}$. The
number of MCMC iterations $N_{MCMC}$. 
\item \textbf{Output.} Predictive value from SBT at location $x_{*}$. 
\end{itemize}
\begin{enumerate}
\item For $k$ in $1:N_{MCMC}$ do 
\begin{enumerate}
\item Sample $u_{*}\sim\text{Unif(0,1)}$ and make the auxilliary point
$(x_{*},u_{*})$. %
\item Input $u_{*}$ to the regression tree $\mathcal{T}_{u}$ with $u_{*}$
as input; %
Denote
the terminal node where $u_{*}$ falls as $\eta_{u_{*}}$ 
\item Input $x_{*}$ to the regression tree $\mathcal{T}_{x\mid\eta_{u_{*}}}$
with $x_{*}$ as input; $\hat{y}(x_{*})=\mathcal{T}_{x\mid\eta_{u_{*}}}(x_{*})$
as output. 
\item Append $\hat{y}(x_{*})$ as predicted value to the prediction sample sequence. 
\end{enumerate}
\item Take the average of prediction sample sequence as predictive value
at location $x_{*}$. 
\end{enumerate}
\rule[0.5ex]{1\columnwidth}{1pt}

\caption{\label{alg:RST-Algorithm-Predict}Sharded Bayesian regression trees predicting algorithm.}
\end{algorithm}

%
%
%
%
%
%
%
%

%
%
%
%
%
%
%
%
%
%
%
%
%
%
%
%
%
%
%
%
%
%
%
%
%
%
%
%
%
%
%
%
%
%
\subsection{Perspectives on BART Ensembles}\label{sec:perspectives on ensemble}

Since we are building SBT as a Bayesian model, more insightful perspectives
are needed. The following two perspectives are two sides of one coin,
depending on whether we want to study the empirical measure derived
from individual sharded trees (``as it is'') or we want to infer
the marginalized measure for the tree population first, and then take
inference. Both perspectives happen on the space of tree structures,
precisely for the tree structure $\mathcal{T}_{u}$, admitting the
same algorithms above.

Perspective one is to treat the ensemble model as a weighted aggregation.
In this perspective, the weights determined by the tree $\mathcal{T}_{u}$
does not marginalize this tree structure out but treat this fixed
structure as if it is the MAP of the tree posterior. This perspective
suggests that the sharded models are individual but not independent,
and an aggregated model produced by (re-)weighting of these representatives
would do us a better job. 

Perspective two is to treat the ensemble model as averaging over the
probability distribution defined by the (normalized) weights. In this
perspective, the weights determined by the tree $\mathcal{T}_{u}$
are considered as an empirical approximation to a marginalized model, where
$\mathcal{T}_{x\mid\eta_{u}}$'s are fitted as new tree models %
and we want to marginalize the
effect of sharding introduced by $\mathcal{T}_u$. If we use regression tree as an ``interpretable''
decision rule in the sense of \citet{rudin2022interpretable}, then
this perspective ask us to derive the decision rule by coming up with
just one margnalized rule: for a new hypothetical $x^{*}$, run through
the marginalized measure and compute the mean model, then we have
$1$ result as our final rule.

\subsection{Weights of sub-models}

Although we have provided two different perspectives on the ensemble
model, we have not yet described the effect of our weights. There are
two lines of literature we briefly recall below.

First, from the generalized additive model literature \citep{mccullagh2019generalized,strutz2011data},
the choice of weights are motivated by reducing the uncertainty in
prediction or the heteroscedasticity in observations based on the
design \citep{cressie1985fitting}. For example, in the weighted least
square regression, the weights are chosen to be inverse proportional
to the location variances, meeting D-optimality in the regression setting.

Second, from the model aggregation literature, research has focused on how
to improve the accuracy of prediction \citep{barutccuouglu2003comparison}
and introduce adaptive and dynamic weighting to improve the overall
accuracy of the ensemble \citep{kolter2007dynamic}. And more recent works
formulate the choice of weights in regression ensembles as an optimization
meta-problem to be solved \citep{shahhosseini2022optimizing}.

Under the assumption of a $B$-expected optimal design, we can achieve a uniform posterior rate for the aggregated model when the underlying function has homogeneous smoothness. This echoes the claim
by \citet{pmlr-v89-rockova19a} that the Galton-Watson prior would ensure a nearly
optimal rate not only for single BART but also in an aggregated scheme.
We restate their main result using our notations below and provide a 
brief intuitive explanation afterwards. 
\begin{thm}
\label{thm:(Theorem-7.1-in-Rockva-Saha-2019)}(Theorem 7.1 in \citet{pmlr-v89-rockova19a}
on the posterior concentration for BART) Assume that the ground-truth $f_{0}$ is $\nu$-Holder
continuous with $\nu\in(0,1]$ where $\|f_{0}\|_{\infty}\apprle\log^{1/2}n$.
Denote the true function as $f_{0}$ and the BART model based on an
ensemble $\mathcal{E}$ as $f_{b}$, and the empirical
norm $\|f\|_{n}\coloneqq\frac{1}{n}\sum_{i=1}^{n}f(\bm{x}_{i})^{2}$.

Assume a regular design $\bm{X}=\{\bm{x}_{1},\cdots,\bm{x}_{n}\}\subset\mathbb{R}^{d}$
(in the sense of Definition 3.3 of \citet{rockova_posterior_2019})
where $d\lesssim\log^{1/2}n$. Assume the BART prior with the number
of trees $T$ fixed, and with node $\eta$ splitting probability  $p_{split}(\eta)=\alpha^{\text{depth}(\eta)}$
for $\alpha\in[1/n,1/2)$. With $\varepsilon_{n}=n^{-\alpha/(2\alpha+d)}\log^{1/2}n$
we have 
\begin{align*}
\prod\left(\left.f_{b}\in\mathcal{F}:\left\Vert f_{0}-f_{b}\right\Vert _{n}>M_{n}\varepsilon_{n}\right|\bm{y}_{n}\right) & \rightarrow0,
\end{align*}
for any sequence $M_{n}\rightarrow0$ in $P_{f_{0}}$-probability,
as the sample size $n$ and dimensionality $d\rightarrow\infty$. 
\end{thm}

The functional space $\mathcal{F}$ of step-functions is defined the same as in \cite{rockova_posterior_2019}.
The regularity assumption for the design points aims at ensuring that
the underlying true signal $f_{0}$ can be approximated by the BART
model. Based on this result, we can extend the posterior concentration
to the sharded-aggregation model. And then from this result below,
we know the choice of weights would affect this rate. 
\begin{thm}
\label{thm:posterior_proof_thm}Under the same assumptions and notations
as in Theorem \ref{thm:(Theorem-7.1-in-Rockva-Saha-2019)}, we suppose
that the sharding tree $\mathcal{T}_{u}$ partitions the full dataset
$\bm{X}=\{\bm{x}_{1},\cdots,\bm{x}_{n}\}\subset\mathbb{R}^{d}$, $\bm{y}_{n}$
into $B$ shards $\mathcal{X}_{b},b=1,\cdots,B$ and corresponding
responses $\bm{y}_{b,n_{b}},\cup_{b=1}^{B}\bm{y}_{b,n_{b}}=\bm{y}_{n}$.
We designate a set of weights $w_{1},\cdots,w_{b}$ whose sum
$\sum_{b=1}^{B}w_{b}=1$ for fixed number of shards $B$. Then,
our sharded-aggregation model would take the form of $\sum_{b=1}^{B}w_{b}f_{b}$
where $f_{b}$ is the BART based on the shard $\bm{y}_{b,n_{b}}$. 
\begin{align}
 & \prod\left(\left.\left\Vert f_{0}-\sum_{b=1}^{B}w_{b}f_{b}\right\Vert _{n}>M_{n}B\cdot\varepsilon_{n}\right|\bm{y}_{n}\right)\label{eq:main_bound}\\
 & \leq B\cdot\max_{b=1,\cdots,B}\prod\left(\left.f_{b}\in\mathcal{F}:\left\Vert f_{0}-f_{b}\right\Vert _{n_{b}}>w_{b}^{-1}\cdot M_{n}\varepsilon_{b,n}\right|\bm{y}_{b,n_{b}}\right)\leq B\varepsilon_{n}.\nonumber 
\end{align}
Then, the aggregation of sharded tree posteriors also concentrate
to the  the ground-truth $f_{0}$. 
\end{thm}

\begin{proof}
See Appendix \ref{sec:Proof-to-Theorem-posterior}. 
\end{proof}
For single tree regression, we motivate the optimal choice of sharding
sizes to be inversely proportional to the probability mass contained
in the sharded region (see Proposition \ref{prop:general optimal}),
which aligns with the spirit of the aforementioned literature. As
a Bayesian model, the posterior contraction rate is more or less a
key element in ensuring the accuracy of posterior prediction. %

The following corollary follows from the exchangeability of data shards among different sub tree model, %
when each sub-model is based on
sharding induced by $\mathcal{T}_u$. The rationale for the corollary is that the RHS of \eqref{eq:main_bound} is a maximum bound that depends
on $\varepsilon_{n}\coloneqq\min_{b=1,\cdots,B}\varepsilon_{b,n}$ and $\varepsilon_{b,n}\coloneqq n_{b}^{-\alpha/(2\alpha+d)}\log^{1/2}n_{b}$ 
subject to $\sum_{b=1}^{B}n_{b}=n$. %

Regardless of the choice of $M_n$, the $\varepsilon_{n}$ is minimized  
when all $w_{b}^{-1}\cdot\varepsilon_{b,n}$ are ``as close to each other''
as possible, following a similar argument like Lemma \ref{lem:maximize_CTOD}, and we know that $w_{b}^{-1}\cdot \varepsilon_{b,n}$ need to be all equal as well in order to minimize RHS. 
\begin{cor}
Under the same assumptions as in Theorem \ref{thm:posterior_proof_thm},
the RHS of \eqref{eq:main_bound} is minimized if the products of
weights and shard sizes,  $w_{b}^{-1}\cdot n_{b}^{-\alpha/(2\alpha+d)}\log^{1/2}n_{b}$,  are all equal. 
\end{cor}

Immediately, this asks us to put weights that are roughly proportional to
$\varepsilon_{b,n}$, and when the $\mathcal{T}_{u}$ satisfies $B$-expected
optimal design when $U$ is uniform, $n_{b}$ are ``roughly equal'', and therefore the
optimal design would require us to put roughly equal weights onto
each shard BART. In other words,  $\mathcal{T}_{u}$ having a $B$-expected
optimal design is a sufficient and necessary condition for equally-weighted
aggregation to have optimal posterior concentration rate when the underlying function is homogeneously smooth.

This also indicates that if it happens that we need to create shards
of different sizes (e.g., we have different amount of computational
resources for different shards), we can adjust the weights for each
shards $w_{b}(n)$ (as a function of sample size $n$) to account for different growth rates.

\section{Experiments and Applications\label{sec:Experiments-and-Applications}}

We first focus on examining the empirical performance based on
 simulations, where our dataset is drawn from synthetic functions
\emph{without} any amount of noise.
We focus on varying the parameter $\mathtt{shardepth}$ which defines the depth of $\mathcal{T}_u$, and therefore the (maximum) number of shards (see Appendix \ref{sec:API_SBT}).

For clarity in figures, we use letters A to F associated with the
boxplots for BART and SBT comparisons in the current section.
For A, B and C, we fit and predict using BART with 25\%, 50\%
and all of the training set as baseline models. For D, E and F, we fit
our SBT model using full training set but different combinations of
model parameters $\mathtt{shardepth}=0,1,2$. Precisely, we use: A:
BART with 25\% training set, B: BART with 50\% training set, C: BART
with full training set, D: SBT with $\mathtt{shardepth}=0$, E: SBT
with $\mathtt{shardepth}=1$, F: SBT with $\mathtt{shardepth}=2$.
We summarize the trade-off between model complexity and goodness-of-fit using out-of-sample RMSE and coverage, and also show that the SBT
model is less computationally expensive compared to the standard BART model. 

\begin{figure}[ht]
\begin{centering}
\includegraphics[width=0.9\textwidth]{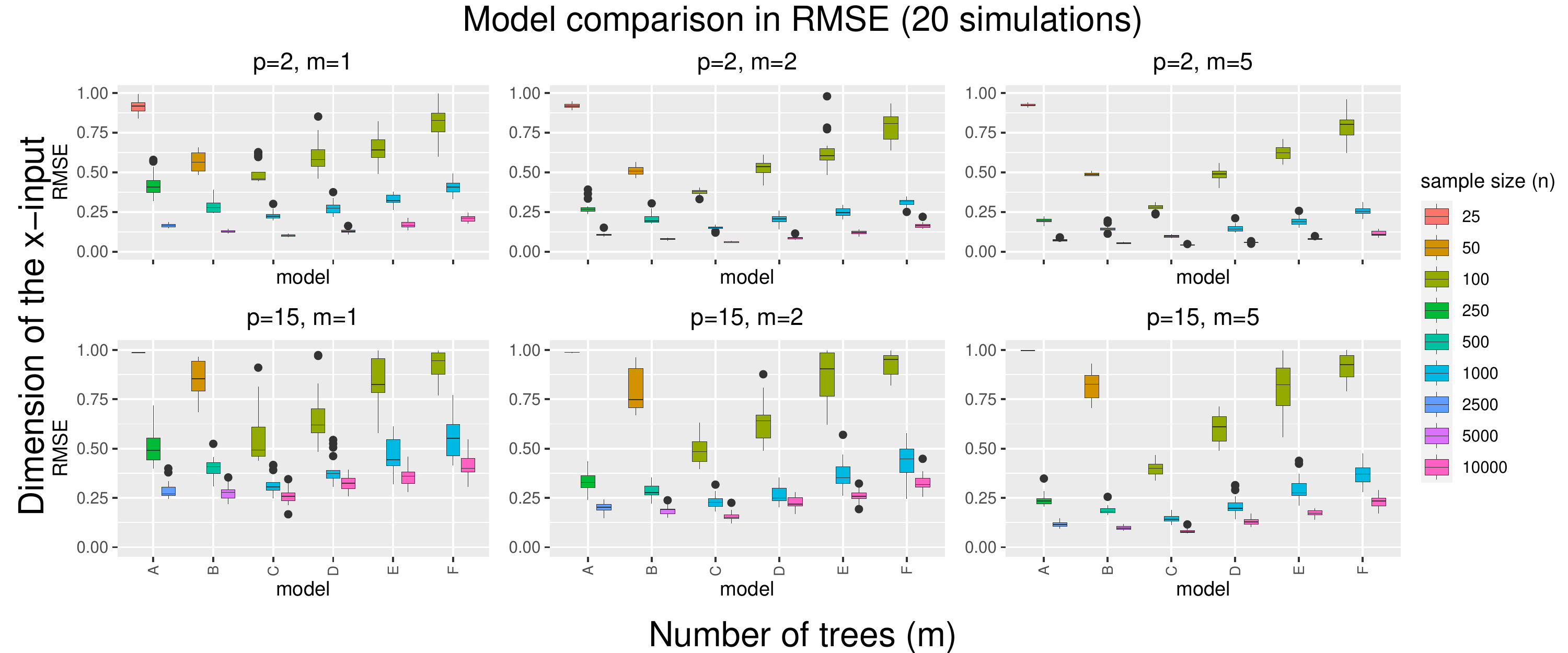}\\
 \includegraphics[width=0.9\textwidth]{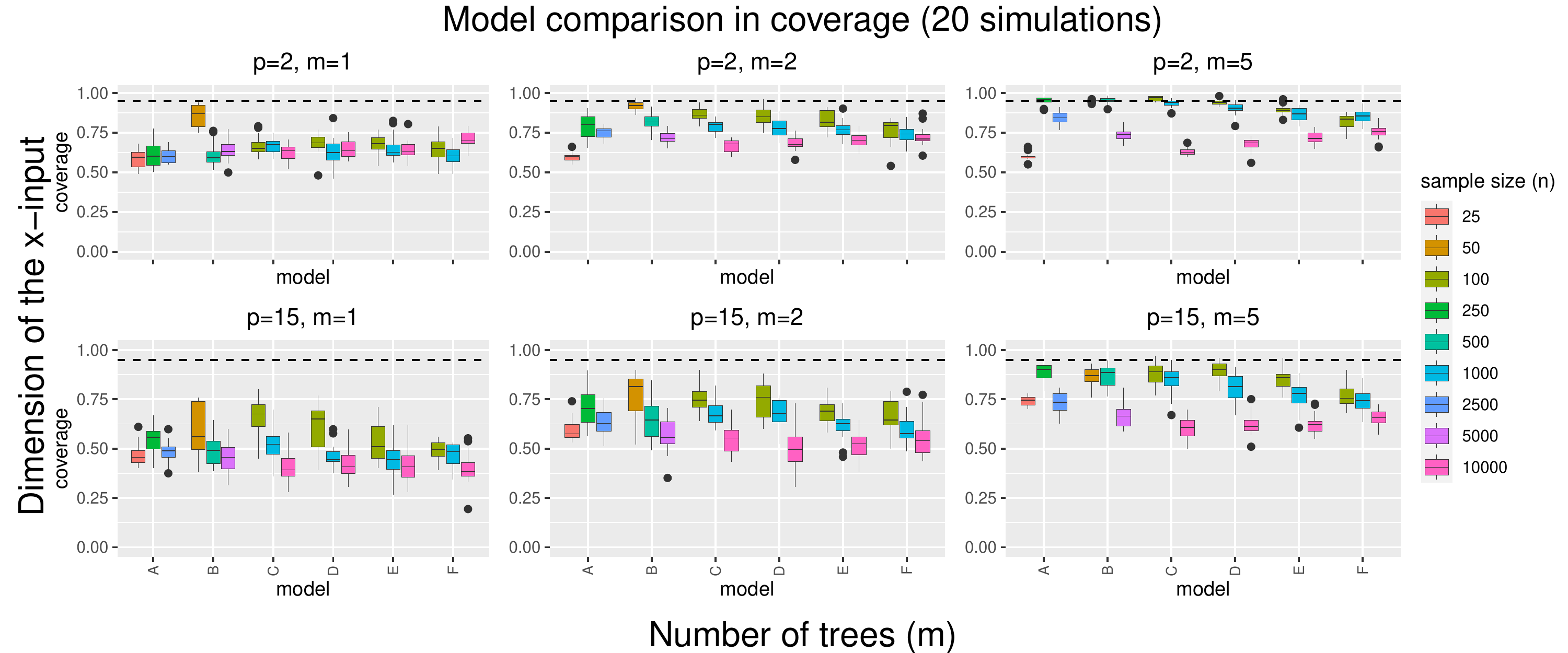}\\
 
\par\end{centering}
\caption{Comparison (in RMSE and 95\% (indicated by dashed line) confidence
interval coverage) between standard BART (with 100\%, 50\% and 25\%
of all samples) and our sharded BART model with different $\mathtt{shardepth}$(=0,1,2).
We design two experiments using branin function defined on $p=2,15$
dimensional domains (displayed in panel rows); with each kind of tree
model with $m=1,2,5,10,20$ trees (displayed in panel columns). In
each panel, we study the sample of sizes $n=100,1000,10000$.}
\label{fig:branin_comprison} 
\end{figure}

\begin{figure}[ht]
\begin{centering}
\includegraphics[width=0.9\textwidth]{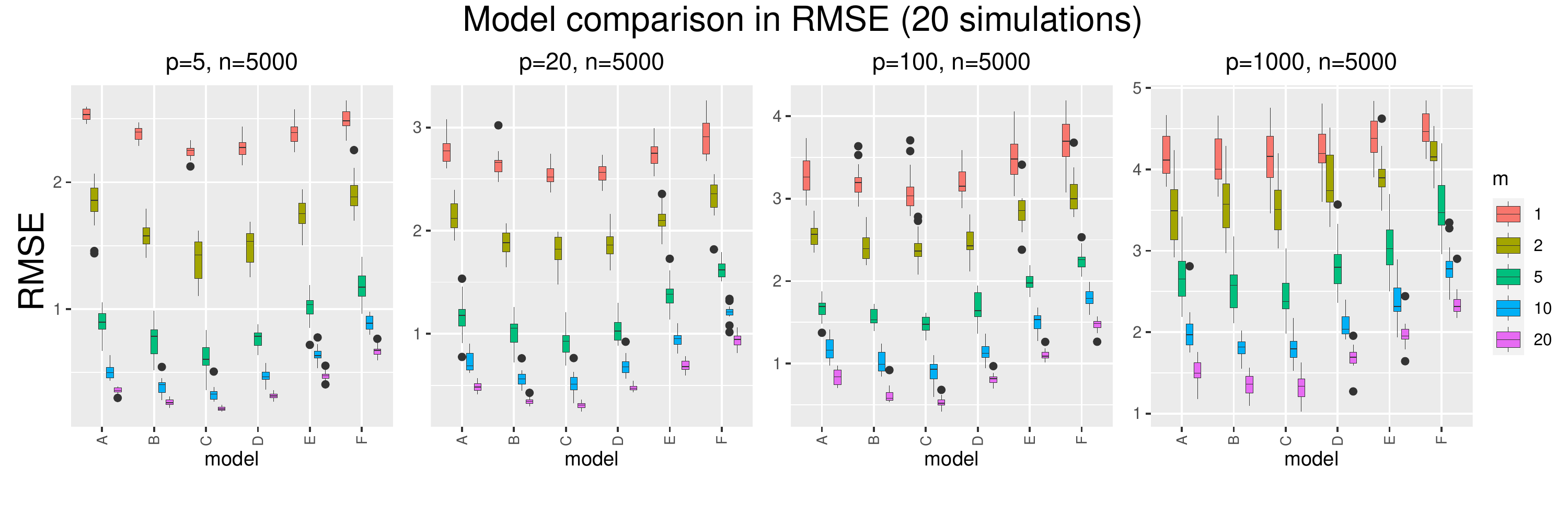}\\
 \includegraphics[width=0.9\textwidth]{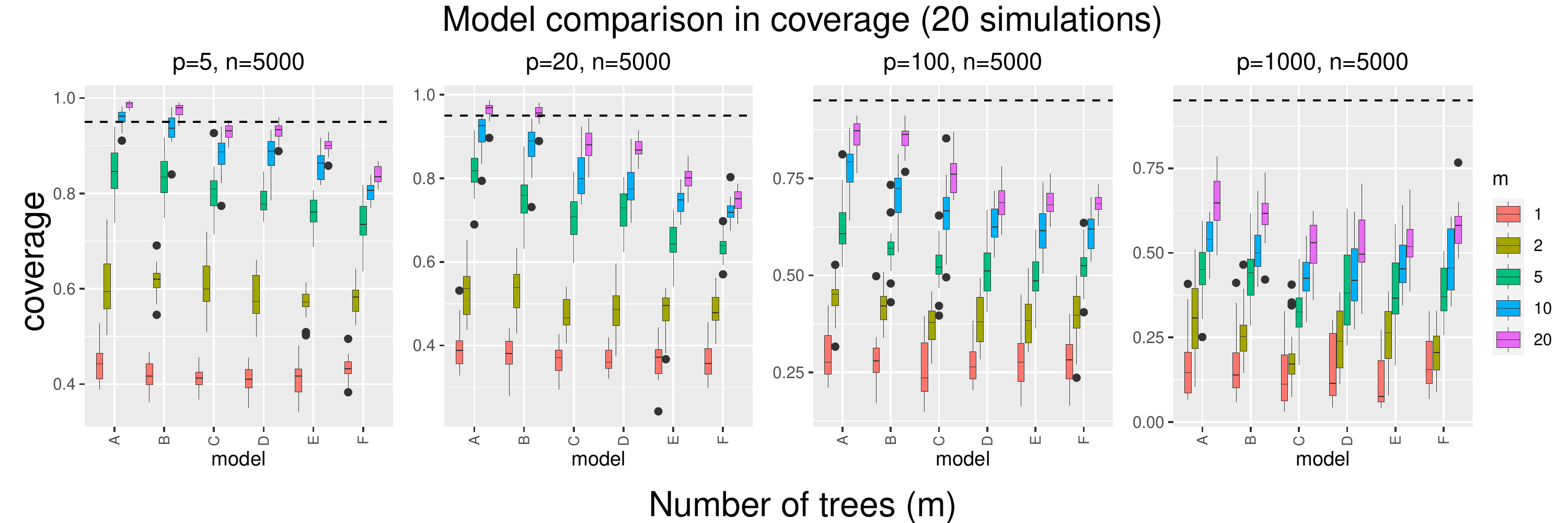}\\
 
\par\end{centering}
\caption{Comparison (in RMSE and 95\% (indicated by dashed line) confidence
interval coverage) between standard BART (with 100\%, 50\% and 25\%
of all samples) and our sharded BART model with different $\mathtt{shardepth}$(=0,1,2).
We design four experiments using $n=5000$ samples from friedman function
defined on $p=5,20,100,1000$ dimensional domains (displayed in panel
columns). In each panel, we study the number of trees $m=1,2,5,10,20$
in each model. }
\label{fig:fried_comprison} 
\end{figure}

\subsection{Small-sample on synthetic functions}

In Figure \ref{fig:branin_comprison}, we present model comparison
based on 20 different fits on the same dataset with different sample
sizes $n$ and dimensionality $p$, drawn from the Branin function defined
on $p=2$ and $p=15$ dimensional domains. For the low dimensional situation ($p=2$),
we compare the SBT with $\mathtt{shardepth}=1$ with BART using $0.25n$ samples, and analogously
for $\mathtt{shardepth}\in\lbrace0,2\rbrace$. This shows that the RMSE of SBT and
standard BART (with the same expected sample size) are  comparable regardless
of the number of trees $m$. Meanwhile, the 95\% confidence interval
coverages decreases with increasing sample size $n$, and increasing
number of trees $m$. However, when $n=10000$, the SBT model has
improved coverage compared to a standard BART of the corresponding
sample sizes. For the high dimensional situation ($p=15$), the trends
observed in the low dimensional situation remains to be true, but
the RMSE is inflated and the coverage is deflated, as expected. 

In Figure \ref{fig:fried_comprison}, we present model comparison
based on 20 different fits on the same dataset with the same sample
size $n=5000$ but different dimensionality , drawn from the Friedman
function defined on $p=5,20,100,1000$ dimensional domains. Here we
fix the sample size, and observe that  SBT generally gives a worse
RMSE compared to the standard BART model, but the difference in RMSE
decreases as dimensionality increases. Increasing the number of trees
in each BART or SBT model will improve the RMSE performance but the
effect of $\mathtt{shardepth}$ is less obvious as $p$ changes.
In terms of 95\% credible interval coverages, the SBT is closer
to standard BART with full sample size $n$, avoiding the inflated
coverage caused by using a smaller sample size $0.5n$ or $0.25n$.

This experiment shows that with small samples, our algorithm could
also work well: (i) ideally $\mathcal{T}_{u}$ shows no sharding
behavior; and  (ii) even with some coarse sharding, the aggregated
posterior is not bad for small samples. %

\subsection{Real-world data: redshift simulation}

\begin{figure}[ht!]
\begin{center}
\includegraphics[height=0.5\textheight]{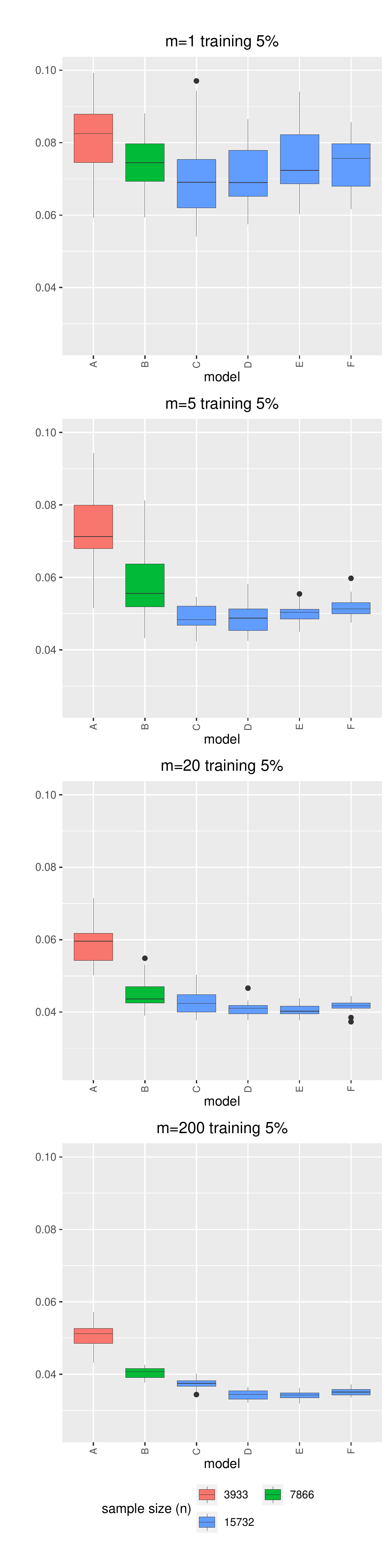}
\includegraphics[height=0.5\textheight]{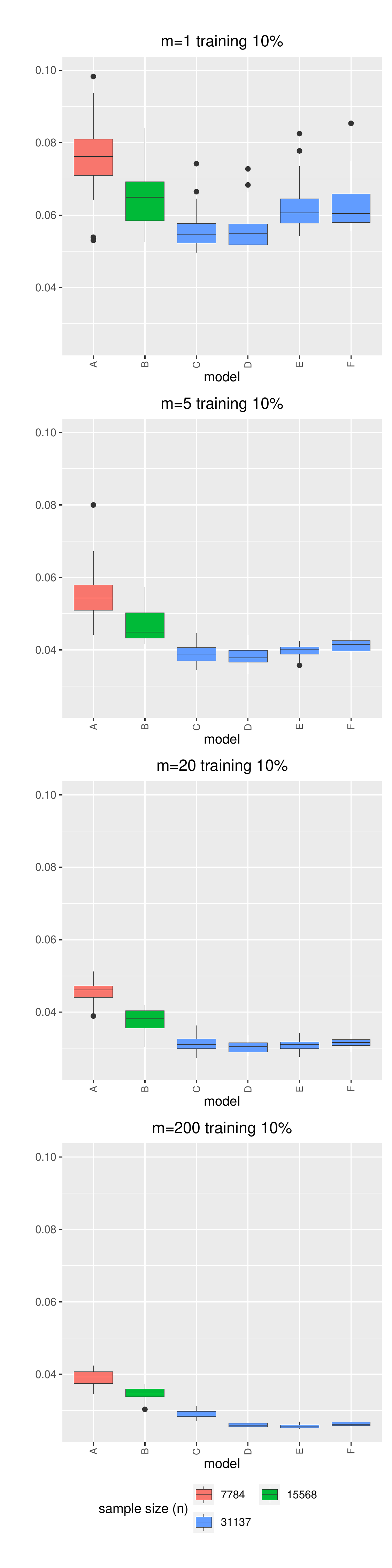}
\includegraphics[height=0.5\textheight]{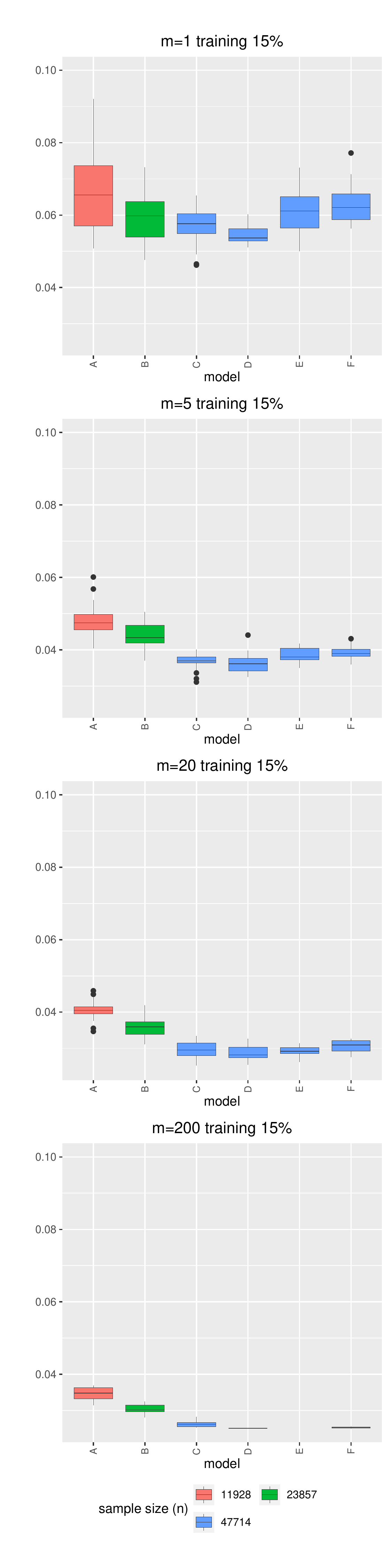}
\includegraphics[height=0.5\textheight]{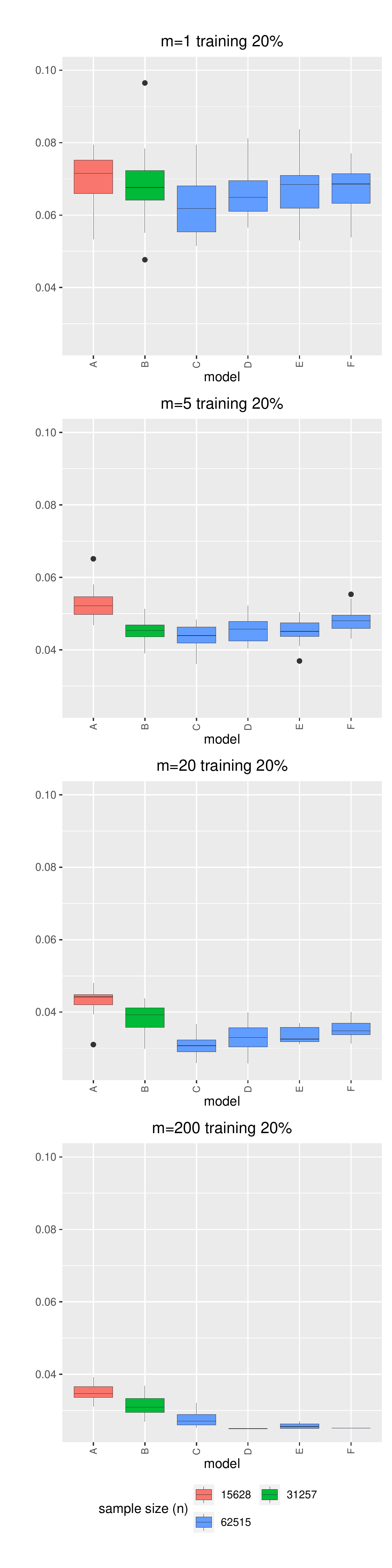}
\includegraphics[height=0.5\textheight]{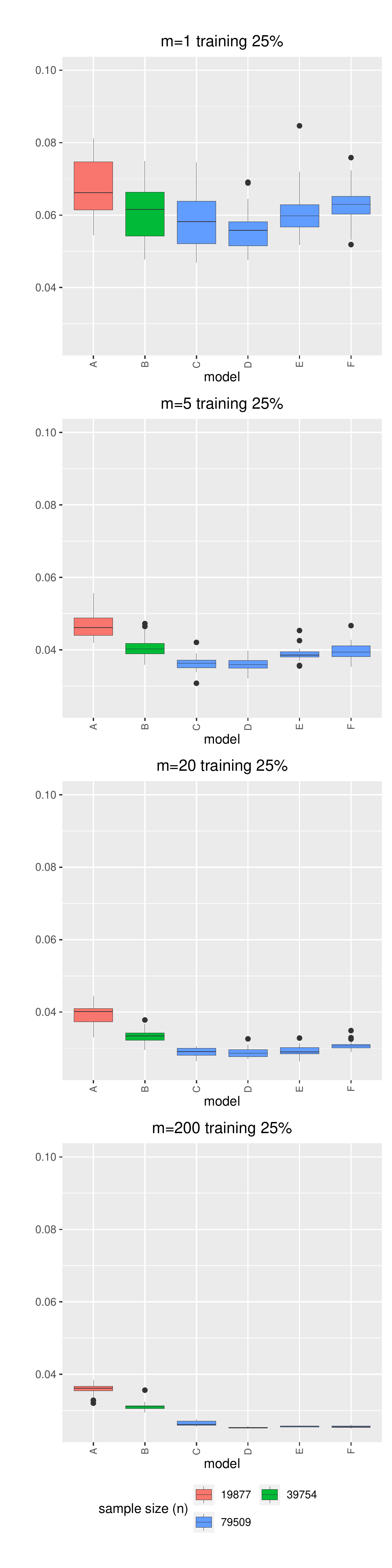}
\end{center}

\caption{\label{fig:redshift_RMSE}Model comparison in terms of RMSE on testing
sets (as complement set of training set from the original redshift
dataset). From the left to the right columns, the training set is taken
as 5\%, 10\%, 15\%, 20\%, 25\% randomly selected samples from the original
dataset; from the top to the bottom rows, the number of trees varies
from $m=1,5,20,200$. In each of the panels, we use boxplots to display
the performance metric distribution from 20 simulations; and color
to denote the actual training set sample sizes corresponding to each
model. }
\end{figure}

\begin{figure}[ht!] 
\begin{center}
\includegraphics[height=0.5\textheight]{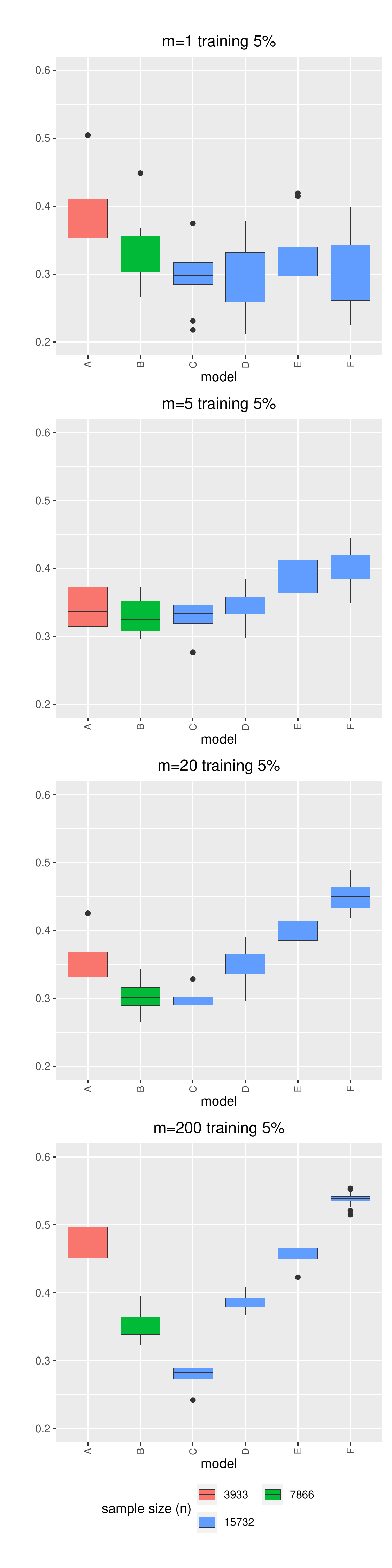} 
\includegraphics[height=0.5\textheight]{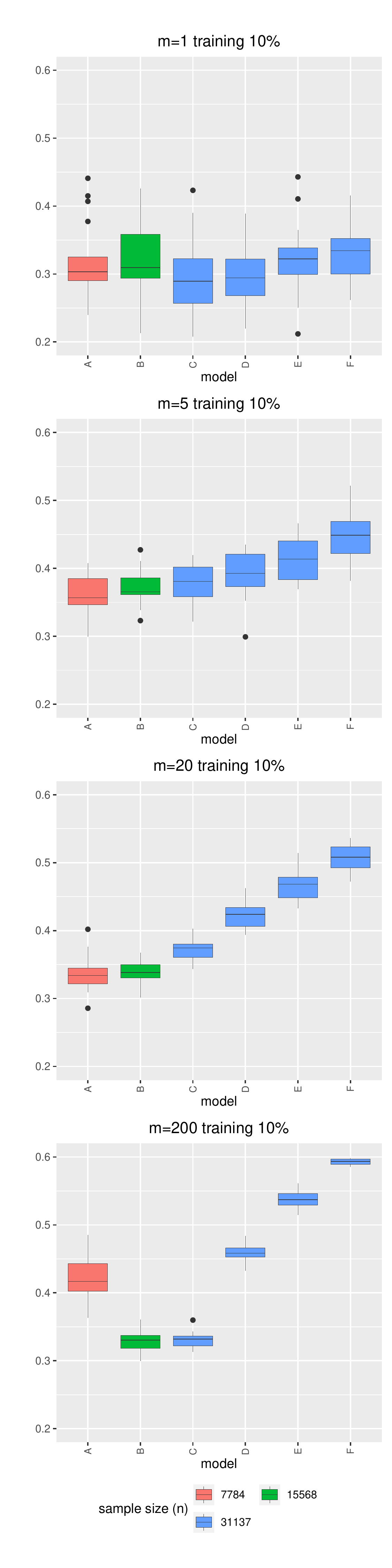}
\includegraphics[height=0.5\textheight]{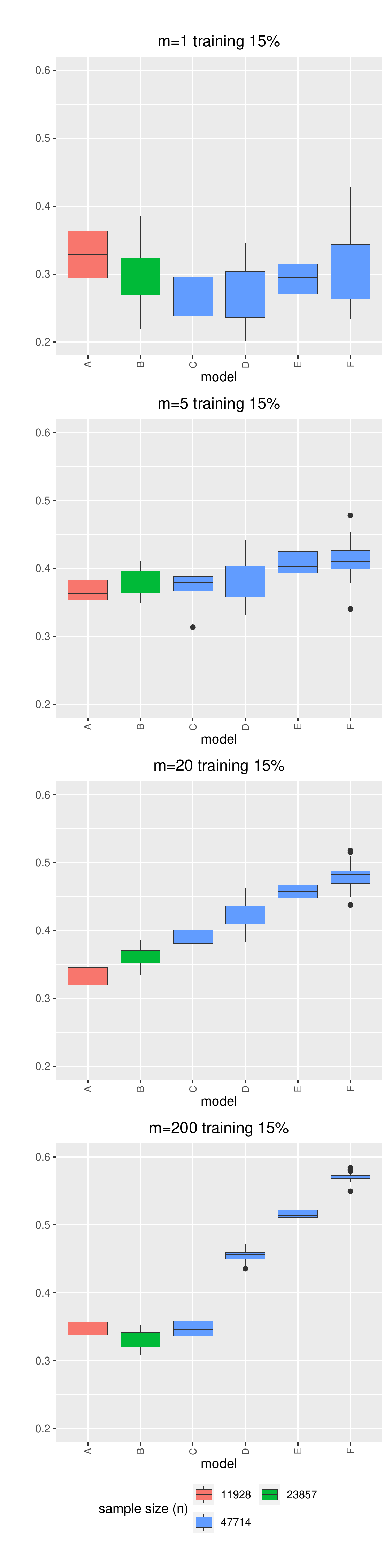}
\includegraphics[height=0.5\textheight]{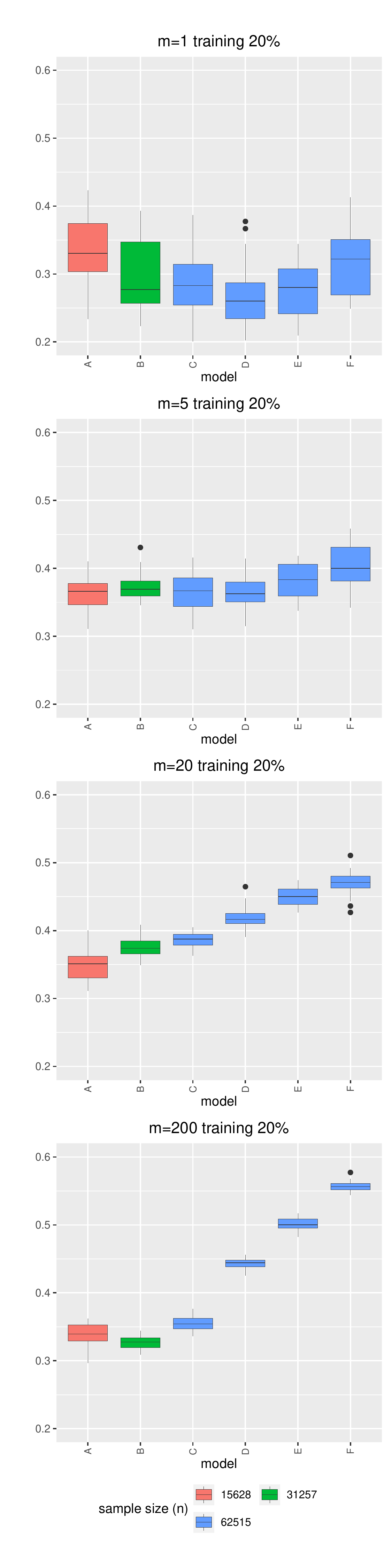}
\includegraphics[height=0.5\textheight]{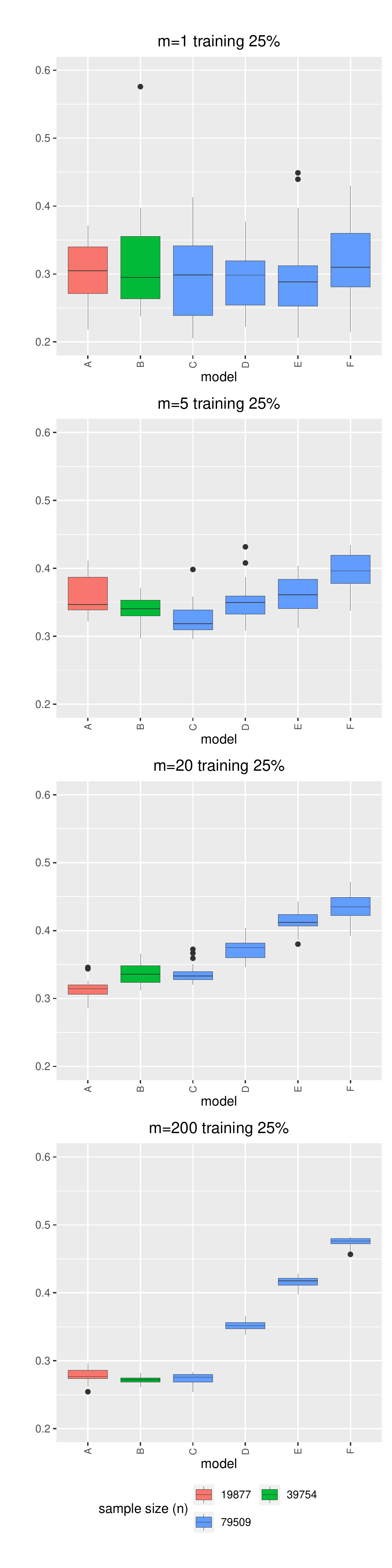}
\end{center}
\caption{\label{fig:redshift_coverage}Model comparison in terms of coverage
of 95\% confidence intervals. From the left to the right columns, the
training set is taken as 5\%, 10\%, 15\%, 20\%, 25\% randomly selected samples
from the original redshift dataset; from the top to the bottom rows,
the number of trees varies from $m=1,5,20,200$. In each of panels,
we use boxplots to display the performance metric distribution from
20 simulations; and color to denote the actual training set sample
sizes corresponding to each model.}
\end{figure}

In this real-world application, we studied a redshift simulator dataset
from cosmology \citep{ramachandra2022machine}. The dataset contains 6,001 functional observations
from a redshift spectrum simulator. For a combination of 7 different
cosmological parameters (e.g., metallicity, dust young, dust old and
ionization etc.), we can observe and collect the spectrum function
as vectors of different lengths (ranging from 48 to 100 points), under
the configuration defined by these set of cosmological parameters.
This presents essential challenges to the modeling task due to its multi-output
nature and observation size.
Typically, one emulates the multivariate vector output depending on
the cosmological parameters using statistical models, in our case BART
and SBT. 

We provide the RMSE for different models in Figure \ref{fig:redshift_RMSE},
where we can see that BART and SBT model with full training sets always
has the overall smaller RMSE. Regardless of the training set size
and model parameters, as soon as the number of trees $m$ is greater
than 2, the SBT has similar RMSE and the difference between BART and
SBT is negligible when $m\geq20$, and SBT is slightly better for larger $m$'s. This ensures that SBT has competitive
estimation and prediction performance compared to BART. %
And the RMSE also slightly decreases as the $m$ increases in SBT and BART.

The more interesting observations come from Figure \ref{fig:redshift_coverage}.
When $m\geq2$, the SBT can have the best performance when the $\mathtt{shardepth}\geq1$,
indicating that a deeper sharding tree $\mathcal{T}_{u}$ leads to better
coverage in confidence intervals. When $m\geq 5$, SBT always
shows better coverage than BART regardless of the choice of other SBT model parameters,
but a deeper sharding tree $\mathcal{T}_{u}$ still presents the best
coverage. With all the rest model parameters fixed, the coverage also
increases significantly as $m$ increases in SBT; but not so obviously
in BART.

From this analysis, we confirm that SBT can fit as well as BART but
provides much better fitting efficiency and coverage with deeper $\mathcal{T}_{u}$
on real-world large datasets.

\subsection{Complexity Trade-offs}

The model complexity \citep{kapelner2013bartmachine,bleich2014bayesian}
and computational complexity \citep{pratola2014parallel} are two
important topics in BART and its variants. In the original BART
algorithm, the complexity can be derived as follows. Our result is
parallel to the computational complexity for spanning trees as shown
in \citet{luo2021bast}.

From the previous experiments, we show how the depth of the tree $\mathcal{T}_{u}$
in SBT trades off with goodness-of-fit of each $\mathcal{T}_{x\mid\eta}$.
Fixing the structures of (each of) the trees $\mathcal{T}_{x\mid\eta}$,
when we have finer partitions induced by $\mathcal{T}_{u}$, the shard
model fit would become worse because each shard tends to have less
data. However, shards become smaller and hence speeds up the model
fitting, until it hits the communication cost bottleneck. 
\begin{lem}
\label{lem:(Computational-complexity-for}(Computational complexity
for BART) Assuming that there are at most $k_{i}$ nodes for the $i$-th
tree for $i=1,\cdots,m$ trees, the worst-case computational complexity
for each MCMC step in BART is $\mathcal{O}(m^{2}\cdot n^2+\sum_{i=1}^{m}k_{i})$. 
\end{lem}

\begin{proof}
See Appendix \ref{sec:Proof-of-Lemma}. 
\end{proof}
This worst-case complexity does grow with the sample size $n$, %
but theoretically there is not an explicit relation between $\max_{i}n_{i}$,
$B$ and $n$ in BART. Based on this lemma, we have the following
result directly from the model SBT construction. 
\begin{lem}
\label{lem:n_k_2n}Suppose that there are $k$ nodes for a complete binary tree $\mathcal{T}$, and
there are $n$ samples fitted to the tree structure. The number of
all nodes $k$ and the sample size $n$ satisfies $1\leq k\leq2n$. 
\end{lem}

\begin{proof}
The terminal nodes must be non-empty and contain at least one sample.
In a complete binary tree, there can be at most $\log n$ layers, hence
the total number of internal nodes is bounded by $n\cdot\sum_{\eta=0}^{\log n}2^{-\eta}=n\cdot(2-2^{-\log n})\leq2n$. 
\end{proof}
\begin{cor}
\label{lem:(Computational-complexity-for-2}(Computational complexity
for BART, continued) Assuming that there are at most $n_{i}=n$ samples
for the $i$-th tree for $i=1,\cdots,m$ trees, the worst-case computational
complexity for each MCMC step in BART is $\mathcal{O}(m^{2}n^2)$. 
\end{cor}

These two lemmas immediately gives the following complexity result
for SBT model. 
\begin{prop}
\label{prop:SBT-complexity} (Complexity for SBT) Let the $\mathcal{T}_{u}$
in SBT be of depth $k$ (at most $2^{k}$ leaf nodes), and assume
that there are at most $n_{j}$ samples for each of the $m_{j}$ single
trees in the $j$-th BART, the worst-case computational complexity
for each MCMC step in SBT is $\mathcal{O}\left(n^2+\sum_{j=1}^{2^{k}}m_{j}^{2}n^2_{j}\right)$.%
\ Specifically, when we choose $\bm{U}$ uniformly thus creating equal sized shards (in expectation) of size $n_j=n/2^k$ and take the same number $m$ of trees in each BART, the above complexity for SBT becomes $\mathcal{O}\left(n^2+m^2\cdot n^2/2^k\right)\rightarrow \mathcal{O}(n^2)$ as $k\rightarrow\infty $.
\end{prop}
If we take equal sized shards and the same number of trees, it follows directly from the proposition that the depth $k$ of $\mathcal{T}_u$ controls the complexity of SBT, and when $k=0$ it reduces to BART complexity as expected. %
We can see that SBT is strictly better than BART in terms of complexity, but shallower $\mathcal{T}_u$ has less complexity advantage. 
However, the big advantage of SBT is that updates at and below each leaf node of $\mathcal{T}_u$ can be done in parallel and that implementation reduces the complexity of SBT to $\mathcal{O}\left(n^2+\max_{j}m_{j}^{2}n^2_{j}\right)$. 

\section{Discussion\label{sec:Discussion}}

In this paper, we introduce the (randomized) sharded Bayesian tree
(SBT) model, which is motivated from inducing data shards from the
optimal design tree $\mathcal{T}_{u}$ and assemble sub-models $\mathcal{T}_{x\mid\eta_{u}}$'s
within the same model fitting procedure. This idea describes a principled
way of designing a distribute-aggregate regime using tree-partition
duality. The closest relative in the literature is \citet{chowdhury2018parallel}
and \citet{scott_bayes_2016}. Compared to \citet{chowdhury2018parallel},
we did not parallelize at the MCMC level but instead at the model level and therefore we do not
have to enforce additional convergence conditions. Compared to \citet{scott_bayes_2016},
we propose a new model that links the optimal sharding and convergence
rate in an explicit manner. 

The SBT construction puts the marginalization into the model construction,
showing the additive structure can be leveraged to realize data parallelism
and model parallelism. Our modeling method shares a spirit from both
lines of research. On one hand, we define the partition by an random
measure $\mathbb{P}$, where \citet{chowdhury2018parallel} introduce
auxilliary variables to ensure MCMC convergence. In SBT, we perform
model averaging between shard models which is justified by  optimal designs from $\mathcal{T}_{u}$ and optimal posterior contraction rates of BART,
which is a novel perspective when applied in conjunction with randomization.
However, our averaging is done ``internal to the model'' and maintains
an interpretable partition rule (given by internal nodes of $\mathcal{T}_{u}$)
within this framework. 

Motivated by scaling up the popular BART model to much larger datasets,
we identify and address the following challenges for applying the
distribute-aggregate scheme for the regression tree model from theoretical
perspectives. 
\begin{enumerate}
\item It is unclear how to design data shards $\mathcal{X}_{i}$'s when
the parameter space is not finite dimensional (e.g., trees), and the
previous work focus on improving the convergence behavior \citep{chowdhury2018parallel}.
We answer this question from an optimal design (of $\mathcal{T}_{u}$)
perspective, that is, probability inverse allocation
is optimal. 
\item It is unclear how to choose the weights $w_{i}$ and the number of
shards $B$ that is appropriate to the dataset from a theoretically
justifiable perspective (e.g., reciprocal of variances) \citep{scott_bayes_2016}.
We answer this question by calibrating posterior contraction rates,
stating that for BART, the weights (in SBT) that lead to fastest posterior
convergence are determined by the sample size and the smoothness of
the actual function. 
\item It is unclear how to generalize the distribute-aggregate paradigm
to summarize and possibly reduce the sum-of-tree or general additive
(Bayesian) model \citep{chipman_bart_2010}. We design a novel additive tree model to incorporate this regime within the framework of the BART construction and show better complexity can be attained using SBT. %
\end{enumerate}

The SBT model introduces randomizations into the model fitting and
prediction procedure, and attain strictly lower complexity compared
to the original BART model. 
Our aggregation regime is justified by optimal design for $\mathcal{T}_u$ under uniform distribution and optimal posterior contraction rates. 
As supported by the experiments on simulation
and real-world big datasets, we have not only witnessed its high scalability
but also the improvement in terms of RMSE under appropriate parameter
configurations. The complexity advantage can be practically improved in the parallel setting.

\section*{Acknowledgment }

HL was supported by the Director, Office of Science, of the U.S. Department
of Energy under Contract No. DE-AC02-05CH11231. The work of MTP was supported in part by the National Science Foundation under Agreements DMS-1916231, OAC-2004601, and in part by the King Abdullah University of Science and Technology (KAUST) Office of Sponsored Research (OSR) under Award No. OSR-2018-CRG7-3800.3.  
This work makes use of work supported by the U.S. Department of Energy, Office of Science, Office of Advanced Scientific Computing Research and Office of High Energy Physics, Scientific Discovery through Advanced Computing (SciDAC) program.  We stored our code
at \url{https://github.com/hrluo/}.

\bibliographystyle{chicago}
\bibliography{Tree_ref}

\appendix
\section{\label{distribution example}Distribution of $\bm{U}$}
\begin{figure}[t]
\centering
\includegraphics[width=0.8\textwidth]{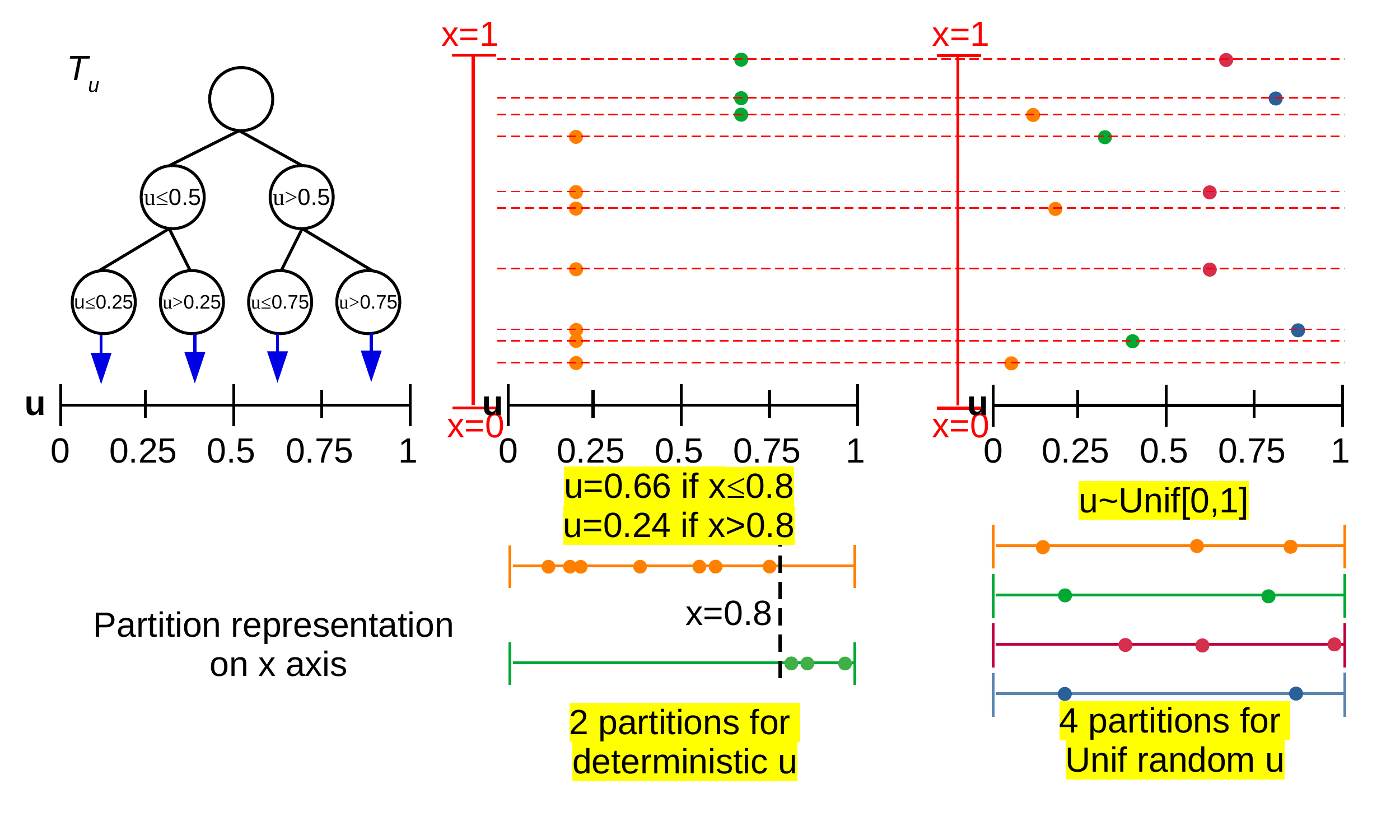} \caption{\label{fig:Balanced-tree-Tu-for-partition}A balanced tree of depth
$k=2$ and its pruned subtree which induced partition on $[0,1]$.
This figure shows how completely deterministic rule for $\bm{U}$ (middle
column) and how uniform random rule for $u$ (right column) create
different sharding on the same dataset $\bm{X}$ (10  dots) on
$[0,1]$.}
\end{figure}
This section shows an example on the effect
of choosing different distributions for the augment random variable
$u$. Suppose that $\mathcal{T}_{u}$ is a fully balanced binary tree
consisting of $B$ terminal nodes. Here we examine two different rules
conditioned on the balanced partition tree $\mathcal{T}_{u}$ at depth
$k=2$ as illustrated in Figure \ref{fig:Balanced-tree-Tu-for-partition}. 
Conditioning on the tree structure $\mathcal{T}_{u}$, if we choose
the uniform variable $\bm{U}\sim\text{Unif}(0,1)$ independent of $\bm{X}$,
then we partition the dataset into shards with approximately equal
sized 4 shards (where the approximation would become exact in the
limit $n\rightarrow\infty$). \\
 They are evenly distributed over $[0,1]$, overlap, and display completely
random designs. %

Conditioning on the same $\mathcal{T}_{u}$ structure, if we choose
the variable $\bm{U}$ determined by $\bm{X}$ in some manner, say using the rule $U=\begin{cases}
0.66 & x\leq0.8\\
0.24 & x>0.8
\end{cases}$, (which can be described by another tree) then we partition the dataset
into 2 shards of sizes 3 and 7 with no intersection. They are supported in $[0,1]$, do not overlap, and displays deterministic
designs. %
If we choose the $\bm{U}$ randomly %
as we will show below, the sharded datasets allocated
to each terminal leaf node can follow an optimal design within the
corresponding partition with a good chance.

\section{\label{sec:Optimal-Design-for}Optimal Design for Constrained Leaves}
\begin{figure}[t]
\centering \includegraphics[width=0.9\textwidth]{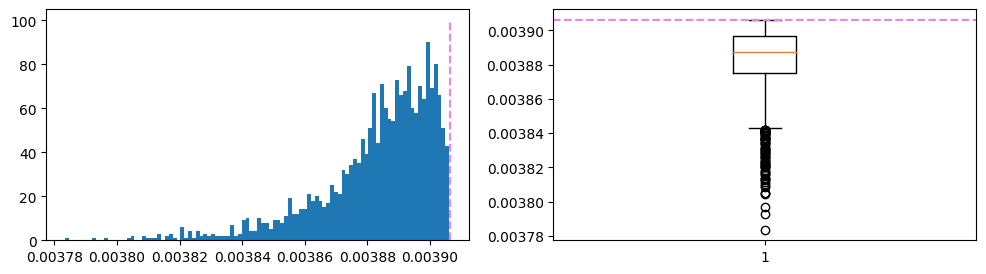}\\
\includegraphics[width=0.9\textwidth]{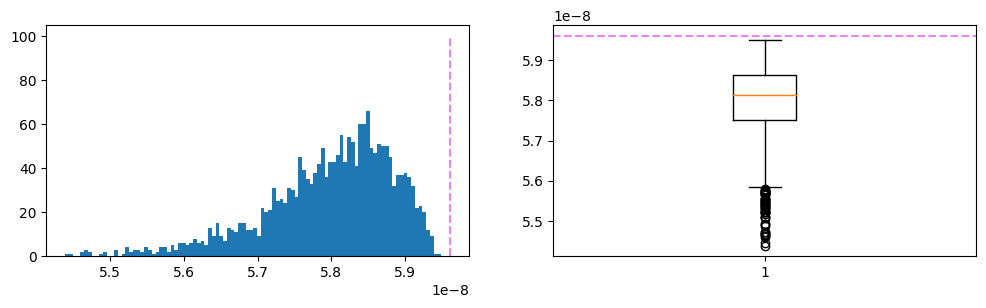} \caption{We generate $n=1000$ random multinomial samples with equally probabilities
$1/B$ being assigned to $B=4$ (top) and $B=8$ (bottom) leaf nodes, and then compute the
$\phi_{n}(\bm{u}_{1},\ldots,\bm{u}_{n}\mid\mathcal{T})$ defined in
\eqref{eq:phi_x_T} for 1000 different multinomial draws and assign them to these leaves. We display 2000 batches of assignments  using histogram (left column) and boxplot (right column). We indicate
the $\phi_{n}^{\max}$ using violet dashed lines in both plots. The
probability of getting at $\phi_{n}^{\max}=\frac{n!}{q!^{(B-r)}(q+1)!^{r}}\frac{1}{B^{n}}$ using the notation in Lemma \ref{lem:maximize_CTOD}.
}
\label{fig:Lemma_pics} 
\end{figure}
We have seen from Figure \ref{fig:Lemma_pics}, that ss the number of leaves $B$ increases, the chance that a random assignment attains optimal design tends to zero.

More generally, we want to solve the optimization problem $\max_{(n_{1},\cdots,n_{B})}\prod_{b=1}^{B}n_{b}$,
with constraints $\sum_{b=1}^{B}n_{b}=n$ and $n_{b}^{-}\leq n_{b}\leq n_{b}^{+}$,
i.e., we need to keep minimum and maximum sample sizes for each leaf.
This kind of scenario is not uncommon, for instance, the tree has
already contain several pilot samples and we want to find the optimal
design conditioned on these pilot samples. Note that this can be solved using integer programming as explained below, and we still assume
that $B\leq n$.

\paragraph*{Case I: Trivial constraints. }

When $n_{b}^{-}\leq\frac{n}{B}\leq n_{b}^{+}$ for all $b$, this
reduces to Lemma \ref{lem:maximize_CTOD}.

\paragraph*{Case II: Nontrivial constraints. }

Otherwise, there are some $b$ such that $\frac{n}{B}\notin[n_{b}^{-},n_{b}^{+}]$.
Then, we consider the equivalent problem 
\begin{align}
\max_{(n_{1},\cdots,n_{B})}\sum_{b=1}^{B}\log n_{b}\text{ s.t. } & \sum_{b=1}^{B}n_{b}=n,\label{eq:constrained_optimal_design}\\
 & n_{b}^{-}\leq n_{b}\leq n_{b}^{+},\nonumber 
\end{align}
where $n_{b}$ are integer indeterminates of the diagonal entries
in the design matrix, and also the number of samples assigned to leaf
$b$. This is closely related to the closest point search problem
on lattice $\mathbb{N}^{B}$\citep{agrell_closest_2002}, where a
polynomial time algorithm is possible when the dimension $B$ is fixed.

However, the problem \eqref{eq:constrained_optimal_design} comes
with non-trivial constraints, and the lattice $\mathbb{N}^{B}$ is
special since its reduced basis is the canonical basis. Therefore,
we take a different approach below.

To begin with, for $B$ indeterminates $n_{1},\cdots,n_{B}$, geometrically
we can perceive the problem as finding a point in 
\[
R\cap\mathbb{N}^{B}\cap C,R\coloneqq\left\{ \left(n_{1},\cdots,n_{B}\right)\mid\sum_{b=1}^{B}n_{b}=n\right\} 
\]
\[
C\coloneqq\left\{ \left(n_{1},\cdots,n_{B}\right)\mid n_{b}^{-}\leq n_{b}\leq n_{b}^{+},b=1,\cdots,B\right\} 
\]
such that it has the closest $L_{2}$ distance (or any convex distance
would work) to the point $\left(\frac{n}{B},\cdots,\frac{n}{B}\right)\in\mathbb{R}^{B}$,
which is possibly not in the $R\cap\mathbb{N}^{B}$.

Then, we can consider the following problem of finding the closest
point $\bm{n}=(\hat{n}_{1},\cdots,\hat{n}_{B})\in C$ on the hyperplane
defined by $\bm{1}\cdot\bm{n}-n=0$, to the continuous solution to
\eqref{eq:constrained_optimal_design}, i.e., $\left(\frac{n}{B},\cdots,\frac{n}{B}\right)$.
The point on the hyper-plane distance that minimizes the followin
$d(\bm{n})\equiv d(\hat{n}_{1},\cdots,\hat{n}_{B})\coloneqq\left(\frac{\left|\sum_{b=1}^{B}\hat{n}_{b}-n\right|}{\sqrt{B}}\right)^{2}+\sum_{b=1}^{B}\left(\hat{n}_{b}-\frac{n}{B}\right)^{2}$,
and we convert the problem into following format, which is usually
easier to solve numerically with fewer constraints (and remove the
strict constraint $\sum_{b=1}^{B}n_{b}=n$). 
\begin{align}
 & \min_{\bm{n}}d(\bm{n}),\text{ s.t., }\bm{n}\in C.\label{eq:constrained_optimal_design-1}
\end{align}

\begin{prop}
A solution to \eqref{eq:constrained_optimal_design-1} in $\bm{n}$,
is a solution to \eqref{eq:constrained_optimal_design}. 
\end{prop}

\begin{proof}
Without loss of generality, we can assume that the integer vector
$\left(\hat{n}_{1},\hat{n}_{2},\hat{n}_{3},\cdots,\hat{n}_{B}\right)$
is a solution to \eqref{eq:constrained_optimal_design-1}, but is
not a solution to \eqref{eq:constrained_optimal_design}. We proceed
by contradiction. Suppose that we only need to perform a $\pm1$ operation
and change it into $\left(\hat{n}_{1}+1,\hat{n}_{2}-1,\hat{n}_{3},\cdots,\hat{n}_{B}\right)$,
to yield a larger value of \eqref{eq:constrained_optimal_design}.
We can focus on the first two coordinates and a difference of mass
1, since all entries are non-negative integers and their sum is fixed
to be $n$, we can perform induction on the number of different entries
and the number of $\pm1$ operations based on this base case. That
means, $\log(\hat{n}_{1})+\log(\hat{n}_{2})<\log(\hat{n}_{1}+1)+\log(\hat{n}_{2}-1)$,
and $\log\left(\frac{\hat{n}_{1}}{\hat{n}_{1}+1}\right)=\log\left(1-\frac{1}{\hat{n}_{1}+1}\right)<\log\left(\frac{\hat{n}_{2}-1}{\hat{n}_{2}}\right)=\log\left(1-\frac{1}{\hat{n}_{2}}\right)$.
This implies $\hat{n}_{2}-\hat{n}_{1}>1\Longrightarrow\hat{n}_{1}-\hat{n}_{2}+1<0$
for integers $\hat{n}_{1},\hat{n}_{2}$, hence the first term in $d(\hat{n}_{1},\hat{n}_{1}+1,\hat{n}_{3},\cdots,\hat{n}_{B})$
remains the same; yet the second term differ by 
\begin{align*}
 & \left[\left(\hat{n}_{1}+1-\frac{n}{B}\right)^{2}+\left(\hat{n}_{2}-1-\frac{n}{B}\right)^{2}\right]-\left[\left(\hat{n}_{1}-\frac{n}{B}\right)^{2}+\left(\hat{n}_{2}-\frac{n}{B}\right)^{2}\right]\\
= & 2\left(\hat{n}_{1}-\frac{n}{B}\right)+1-2\left(\hat{n}_{2}-\frac{n}{B}\right)+1\\
= & 2\left(\hat{n}_{1}-\hat{n}_{2}+1\right)<0,
\end{align*}
contradicting the assumption that $\left(\hat{n}_{1},\hat{n}_{2},\hat{n}_{3},\cdots,\hat{n}_{B}\right)$
is a solution to \eqref{eq:constrained_optimal_design-1}. 
\end{proof}
\begin{figure}
\centering \includegraphics[width=0.35\textwidth]{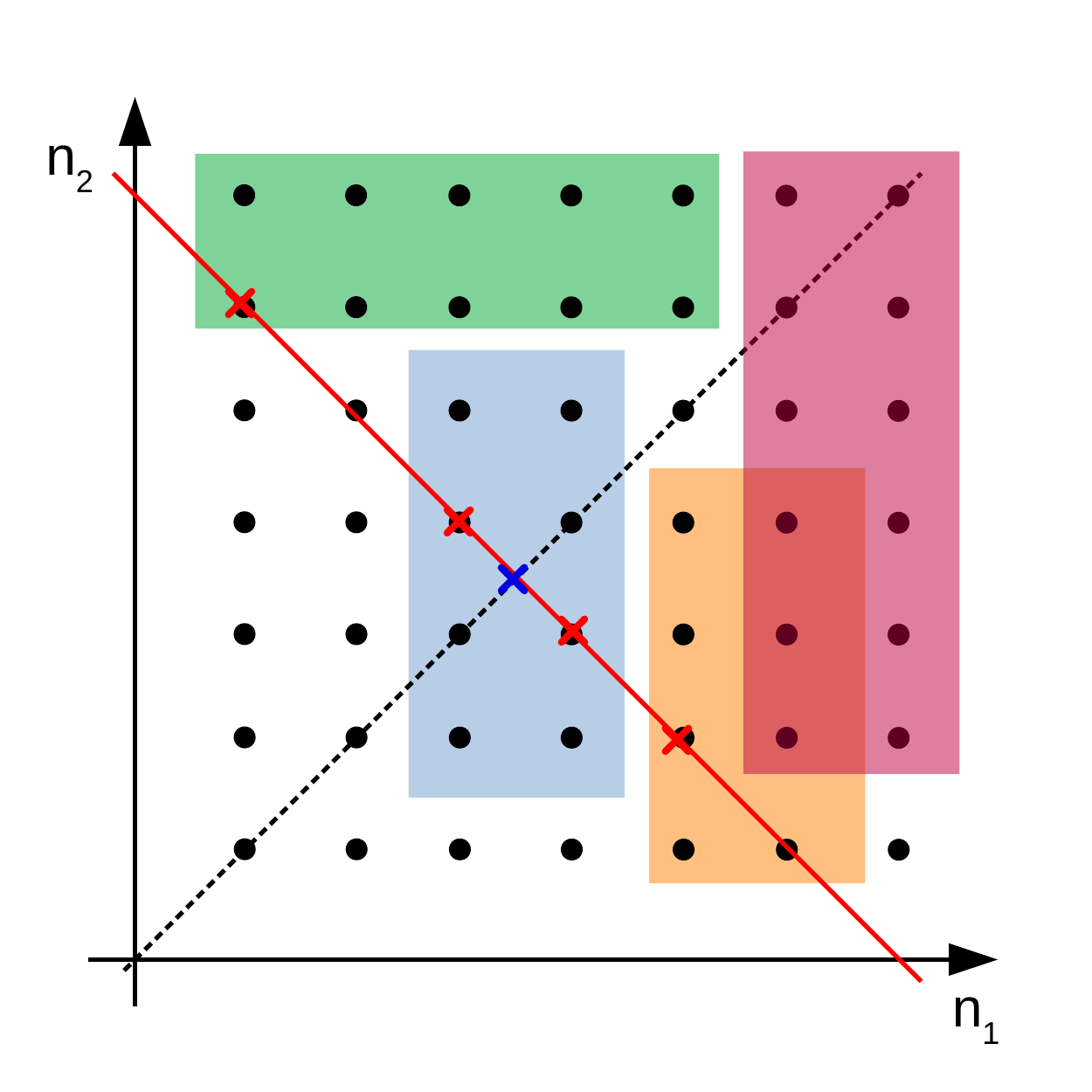}
\caption{Geometric intuition of proof for constrained optimal design. We display
a $\mathbb{N}^{2}$ lattice on the $\mathbb{R}^{2}$ plane, the fixed
sum $n_{1}+n_{2}=n(=7)$ is displayed as red solid line, crossing
all 7 possible pairs $(1,6),(2,5),\cdots,(6,1)$. The continuous solution
to \eqref{eq:constrained_optimal_design} is displayed by the blue
cross, i.e., the pair $(3.5,3.5)$. \protect \protect \protect \\
 If the constrained region (e.g., the blue region covers the $(3.5,3.5)$,
while the green and orange regions do not) and has non empty intersection
with $\mathbb{N}^{2}$ and the hyper-plane $\sum_{b=1}^{B}n_{b}=n$
(i.e., red solid line), then we can find solutions (displayed by red
crosses) along the hyper-plane $\sum_{b=1}^{B}n_{b}=n$. \protect
\protect \protect \\
 If the constrained region (e.g., the violet region) does not intersect
with $\mathbb{N}^{2}$ or the hyper-plane $\sum_{b=1}^{B}n_{b}=n$,
there is no solution.}
\label{fig:AMGMproof} 
\end{figure}

The intuition for $B=2$ of this argument is shown in Figure \ref{fig:AMGMproof}.
Practically, we can run constrained gradient descent with initial
point $\left(\frac{n}{B},\cdots,\frac{n}{B}\right)\in\mathbb{R}^{B}$.

\section{\label{sec:Equivalence-Theorem}Other Optimality Criteria}

This section discusses the D-optimal design, Min-max design and A-optimal
design in the tree regression setting.

Note that this is not a special case of the Kiefer's equivalence theorems,
since they assume the variables are all $\mathbb{R}$. We follow the
exposition Theorem 2.2.1 by \citet{fedorov_theory_2013}. Their proof
techniques cannot be directly applied since we cannot take derivative
when the variables lie in $\mathbb{N}$; although \citet{fedorov_theory_2013}
pointed out that continuous optimal design on $\mathbb{R}$ can approximate
discrete designs on $\mathbb{N}$, we can have more precise results
in tree models. The idea of the proof is similar but instead of linear
combinations of designs we shall consider the operations of adding
1 to one $n_{b}$ and subtract 1 from another $n_{b'}$. 
\begin{thm}
\label{thm:Assume-Equivalence}Assume that the number of sample $n=q\cdot B+r,r<B\leq n,q\in\mathbb{N}$
where $B$ is the number of leaf nodes in a fixed tree $\mathcal{T}$.
Then, the following claims implies the next (i.e., $1\Rightarrow2\Rightarrow3$): 
\end{thm}

\begin{enumerate}
\item The D-optimal design for a tree $n_{1},n_{2},\cdots,n_{B}$ maximizes
$\phi_{n}(\bm{x}_{1},\ldots,\bm{x}_{n}\mid\mathcal{T})=\prod_{b=1}^{B}\frac{n_{b}}{n}$ 
\item The Min-max design for a tree $n_{1},n_{2},\cdots,n_{B}$ minimizes
\[
\phi_{n}^{M}(\bm{x}_{1},\ldots,\bm{x}_{n}\mid\mathcal{T})\coloneqq\max_{\bm{u}}\bm{f}(\bm{u})^{T}\text{diag}\left(n_{1}^{-1},n_{2}^{-1},\cdots,n_{B}^{-1}\right)\bm{f}(\bm{u})
\]
where $\bm{f}(\bm{u})\coloneqq\left(f_{1}(\bm{u}),\cdots,f_{B}(\bm{u})\right)^{T}$. 
\item The design criteria for a Min-max design $\bm{x}_{1}^{*},\ldots,\bm{x}_{n}^{*}$
is $\phi_{n}^{M}(\bm{x}_{1}^{*},\ldots,\bm{x}_{n}^{*}\mid\mathcal{T})=q^{-1}$. 
\end{enumerate}
\begin{proof}
($1\Rightarrow2$) The claim of 2 follows from the usage of Lemma
\ref{lem:maximize_CTOD}. We examine the definition of the Min-max
design here and note that the $f_{b}(\bm{u})$ are actually indicator
functions that indicates whether $\bm{u}$ lies in the domain defined
by the $b$-th leaf node. Therefore, $\max_{\bm{u}}\bm{f}(\bm{u})^{T}\text{diag}\left(n_{1}^{-1},n_{2}^{-1},\cdots,n_{B}^{-1}\right)\bm{f}(\bm{u})=\max\left(n_{1}^{-1},n_{2}^{-1},\cdots,n_{B}^{-1}\right)$
and this maximum is attain by $\bm{u}=(0,\cdots,0,1,0,\cdots,0)^{T}$
with 1 at the index where the diagonal term is the largest. To minimize
this $\phi_{n}^{M}(\bm{x}_{1},\ldots,\bm{x}_{n}\mid\mathcal{T})$,
we need to minimize the $\max\left(n_{1}^{-1},n_{2}^{-1},\cdots,n_{B}^{-1}\right)$,
in other words, make sure $n_{1}^{-1},n_{2}^{-1},\cdots,n_{B}^{-1}$
are as close to each other as possible, subject to $\sum_{b=1}^{B}n_{b}=B$.
But we have already derive from Lemma \ref{lem:maximize_CTOD} that
only a permutation of the maximizer of $\phi_{n}(\bm{x}_{1},\ldots,\bm{x}_{n}\mid\mathcal{T})$
would satisfy this requirement.

($2\Rightarrow3$) The claim of 3 follows from a direct computation
and the $\bm{u}=(0,\cdots,0,1,0,\cdots,0)^{T}$ with 1 at the any
one of $r$ indices where the diagonal term is $m$. 
\end{proof}
However, the claim of 1 does not follow from the claim of 3 (i.e.,
$3\not\Rightarrow$1). Without loss of generality, we can assume $n_{1}=q$,
and all the rest leaf counts $n_{2},n_{3},\cdots,n_{B}\geq q$. With
this requirement, we can write $n_{b}=q+\delta_{b}$ for $b=2,\cdots,B$
and find a solution $\delta_{2},\cdots,\delta_{B}\geq0$ that satisfies
$\sum_{b=2}^{B}\delta_{b}=n-q\cdot B=r$. Then any of these solutions
would satisfy $\phi_{n}^{M}(\bm{x}_{1}^{*},\ldots,\bm{x}_{n}^{*}\mid\mathcal{T})=q^{-1}$
but only the solution $\delta_{2}=\cdots=\delta_{r+1}=1$ and $\delta_{r+2}=\cdots=\delta_{B}=0$
(or its permutation) would lead to the optimal design claimed in 1.
Kiefer's argument for continuous design would not work here since
we cannot take convex combination of two designs, and all $f_{b}(\bm{u})$
are either 1 or 0.

The D-optimality is strictly stronger than Min-max optimality in the
tree regression setting. Under the assumption that the minimizer to
$\phi_{n}^{M}(\bm{x}_{1}^{*},\ldots,\bm{x}_{n}^{*}\mid\mathcal{T})=q^{-1}$
being unique, we can even find $m=n_{1}<n_{2}<\cdots<n_{B}$ that
is Min-max optimal yet not D-optimal.

Next, we would consider the \emph{A-optimality}, where the optimal
criteria function (to be minimized) is defined to be 
\[
\phi_{n}^{A}(\bm{x}_{1},\ldots,\bm{x}_{n}\mid\mathcal{T})\propto\text{Trace }\text{diag}\left(n_{1}^{-1},n_{2}^{-1},\cdots,n_{B}^{-1}\right)=\sum_{b=1}^{B}n_{b}^{-1}.
\]

\begin{prop}
Under the assumption of Theorem \ref{thm:Assume-Equivalence}, the
D-optimality is equivalent to A-optimality in the sense that, any
D-optimal solution is A-optimal, vice versa. 
\end{prop}

\begin{proof}
It is straightforward that $\min_{(\bm{x}_{1},\cdots,\bm{x}_{n})}\phi_{n}^{A}(\bm{x}_{1},\ldots,\bm{x}_{n}\mid\mathcal{T})=\min_{(n_{1},\cdots,n_{B})}\sum_{b=1}^{B}n_{b}^{-1}$
subject to $\sum_{b=1}^{B}n_{b}=n$. For our convenience, we first
assume $n_{1},\cdots,n_{B}\in\mathbb{R}$ and construct the Lagrange
multiplier $\mathcal{L}(n_{1},\cdots,n_{B};\lambda)\coloneqq\sum_{b=1}^{B}n_{b}^{-1}+\lambda\left(\sum_{b=1}^{B}n_{b}-n\right)$
and take derivatives (although the function is defined over integer grid $\mathbb{N}^B$, it admits a natural extension to $\mathbb{R}^B$ and justifies the derivation below): 
\begin{align*}
\frac{\partial\mathcal{L}}{\partial n_{b}}=-n_{b}^{-2}+\lambda,\quad & \frac{\partial\mathcal{L}}{\partial\lambda}=\sum_{b=1}^{B}n_{b}-n.
\end{align*}
Set equations to zeros we solve $\hat{n}_{b}=\frac{1}{\sqrt{\hat{\lambda}}},\frac{B}{\sqrt{\hat{\lambda}}}-n=0$
and its Hessian is positive so it is a minimizer. Therefore, these
$\hat{n}_{b}=\left(\frac{n}{B}\right)$ would give us a continuous
solution. To attain a solution in $\mathbb{N}$, it suffices to repeat
the argument in Appendix \ref{sec:Optimal-Design-for} with no constraints
on $n_{b}$, i.e., $n_{b}^{-}=-\infty$, $n_{b}^{+}=+\infty$. This
set of solutions coincide with the solution described in Lemma \ref{lem:maximize_CTOD}. 
\end{proof}
The D-optimality is equivalent to the A-optimality in the tree regression
setting, since they always share the same set of solutions.

\section{\label{sec:Asymptotic-Optimal-Designs}Asymptotic Optimal Designs}

The result in section \ref{sec:OD_in_reg_tree} assumes fixed number
of leaves $B$ and number of samples $n$ (or the depth of the tree
$\kappa$). In the case $n<B$, we take the convention that the design
criterion is the product over non-empty nodes, and hence it is maximized
by allocating 1 sample for $n$ of our  $B$ nodes, and leave the remaining
nodes containing no samples. We do not consider this situation in
the current paper. Following the derivations in \eqref{eq:UTOD_derivation},
we still want to ensure that in the expectation (w.r.t. the designated
probability measure $\mathbb{P}$) the quantity $\mathbb{E}\phi$
is maximized.

Asymptotically, when letting $B\rightarrow\infty$ the dependence
between random variables that count the number of sample points in
different leaves are asymptotically independent, according to \citet{kolchin_random_1978}(Chapter
III and IV). Therefore, the product of first moments would give us
good approximation to the quantity $\mathbb{E}\phi$. When both $n$
and $B$ tend to infinity, the term $\frac{n!}{n^{B}}$ affects the
asymptotic limiting quantity $\mathbb{E}\phi$. Consider a special
case where $B=B(n)=n^{\omega}$ where the number of leaves grows as
a polynomial of the sample size $n$. Then when $0\leq\omega<1$, $\frac{n!}{n^{B}}\rightarrow\infty$;
and when $\omega\geq1$, $\frac{n!}{n^{B}}\rightarrow0$. Unfortunately,
usually the only bound we can ensure is $0\leq B(n)\leq2^{n}$ in
a typical tree regression. This echoes one common difficulty faced
by the partition-based methods: when the resolution grows, the number
of partitions grows exponentially.

Obviously, in the case of finding an optimal design for the regression
tree induced partition, we want to take the tree structure into consideration.
On one hand, we want to regulate the partition so that even for high
dimensional data (i.e., $\bm{X}\subset\mathbb{R}^{d}$ where $d$
is large) there would not be too many partition components induced
by the tree; on the other hand, we usually limit the depth of the
tree to reduce the complexity of the regression model for computational
considerations. This frees the tree model from the exponentially growing
number of partitions for large $d$. In a regression tree induced
partition, we can limit the growth rate of the number of total partition
components and avoid the curse of dimensionality. 

Corresponding to this principle of limiting the depth of a tree, we
consider a slightly different set of notions of optimality criteria.
Analogous to \eqref{eq:phi_T}, we consider the maximal subtree before
depth $\kappa_{0}$: 
\begin{align}
\phi(\mathcal{T}\mid\mathbb{P})\coloneqq\prod_{b\text{ is a node of depth }\kappa_{0}}\mathbb{P}\left(\bm{u}\in R_{b}\right).\label{eq:phi_T_kappa0}
\end{align}
And the criteria \eqref{eq:phi_x_T} can be generalized to the corresponding
``depth constrained'' version by replacing the product term $\prod_{b=1}^{B}$
with $\prod_{b\text{ is a node of depth }\kappa_{0}}$. Instead of
traversing across all leaves, we transverse all internal nodes of
specified depth $\kappa_{0}$. This means that we consider the optimality
only up to $2^{\kappa_{0}}$ nodes. It is possible but not necessary
to make $\kappa_{0}$ depend on the sample size $n$, in order to
take the potential growing model complexity into consideration.

\section{\label{sec:Proof-to-Theorem-posterior}Proof of Theorem \ref{thm:posterior_proof_thm}}

We prove the following:%
\begin{thm*}
Under the same assumptions and notations as in Theorem \ref{thm:(Theorem-7.1-in-Rockva-Saha-2019)},
we suppose that the sharding tree $\mathcal{T}_{u}$ partitions the
full dataset $\bm{X}=\{\bm{x}_{1},\cdots,\bm{x}_{n}\}\subset\mathbb{R}^{d}$,
$\bm{y}_{n}$ into $B$ shards $\mathcal{X}_{b},b=1,\cdots,B$
and corresponding responses $\bm{y}_{b,n_{b}},\cup_{b=1}^{B}\bm{y}_{b,n_{b}}=\bm{y}_{n}$, where $n_b$ is the sample size of each shard.
We designate a set of weights $w_{1}(n),\cdots,w_{B}(n)$ whose
sum $\sum_{b=1}^{B}w_{b}(n)=1$ for fixed number of shards $B$. 
\end{thm*}
\begin{proof}
To achieve this goal, we first consider application of Theorem \ref{thm:(Theorem-7.1-in-Rockva-Saha-2019)}
for each BART fitted using only its shard $\bm{y}_{b,n_{b}}$
for the $\varepsilon_{b,n}\coloneqq n_{b}^{-\alpha/(2\alpha+d)}\log^{1/2}n_{b}$.
By theorem \ref{thm:(Theorem-7.1-in-Rockva-Saha-2019)}, for an arbitrary $\epsilon>0$, we can find a large integer $N_{b,\epsilon}\in\mathbb{N}$
such that, by the definition
of the empirical norm $\|f\|_{n}\coloneqq\frac{1}{n}\sum_{i=1}^{n}f(\bm{x}_{i})^{2}$
with respect to a set $\bm{X}^{T}=(\bm{x}_{1}^{T},\cdots,\bm{x}_{n}^{T})$
of size $n$, 
\begin{align}
 & \prod\left(\left.f_{b}\in\mathcal{F}:\left\Vert f_{0}-f_{b}\right\Vert _{n_{b}}>w_{b}(n)^{-1}\cdot M_{n_{b}}\cdot\varepsilon_{b,n}\right|\bm{y}_{b,n_{b}}\right)\leq\epsilon\label{eq:posterior_proof_3}\\
 & \text{for }n_{b}\geq N_{b,\epsilon}\nonumber 
\end{align}
For now, let us consider the weights $w_{b}(n)^{-1}$ to be  a shard-specific sequence that depends on the total sample size $n$.
Since the $M_{n_{b}}$ can be any sequence that tends to infinity
(as $n_{b}\rightarrow\infty$) as stated in Theorem \ref{thm:(Theorem-7.1-in-Rockva-Saha-2019)},
we can insert $w_{b}(n)^{-1}$ without loss of generality.

Then, define the following auxiliary quantities: 
\begin{align*}
N_{\epsilon} & \coloneqq\max_{b=1,\cdots,B}N_{b,\epsilon},\\
\varepsilon_{n} & \coloneqq\min_{b=1,\cdots,B}\varepsilon_{b,n}=\min_{b=1,\cdots,B}n_{b}^{-\alpha/(2\alpha+d)}\log^{1/2}n_{b},\\
M_{n} & \coloneqq\min_{b=1,\cdots,B}M_{n_{b}}.
\end{align*}
We can take all $n_{1},\cdots,n_{b}\geq N_{\epsilon}$ such that \eqref{eq:posterior_proof_3}
holds for each $f_{b},b=1,\cdots,B$. Since the $\bm{X}$ is splitted into (non-overlapping)
shards of cardinality $n_{1},\cdots,n_{B}$, then we want to replace
the empirical norms $\|\cdot\|_{n_{b}}$ inside the event with $\|\cdot\|_{n}$
as well. To do so, observe that, 
\begin{align}
\|f\|_{n} & \coloneqq\frac{1}{n}\sum_{i=1}^{n}f(\bm{x}_{i})^{2}=\frac{1}{n}\sum_{b=1}^{B}\sum_{\bm{x}_{i}\in\mathcal{X}_{b}}f(\bm{x}_{i})^{2}=\sum_{b=1}^{B}\frac{1}{n}\sum_{\bm{x}_{i}\in\mathcal{X}_{b}}f(\bm{x}_{i})^{2}\nonumber \\
 & \leq\sum_{b=1}^{B}\frac{1}{n_{b}}\sum_{\bm{x}_{i}\in\mathcal{X}_{b}}f(\bm{x}_{i})^{2}=\sum_{b=1}^{B}\|f\|_{n_{b}}\label{eq:triangle_norm_emp}
\end{align}
This inequality $\|f\|_{n}\leq\sum_{b=1}^{B}\|f\|_{n_{b}}$ for $\sum_{b=1}^{B}n_{b}=n$,
also holds for overlapping shards.

Now we are ready to specify the weights in our sharded model, taking
the weighted sum perspective, leads us to consider a set of weights
$w_{n,(1)},\cdots,w_{n,(B)}$ whose sum $\sum_{b=1}^{B}w_{b}(n)=1$
for fixed number of shards $B$. This assumption on the summation
of 1 can be relaxed to finite sum with appropriate normalization using
their sum, and these weights could depend on the sample size $n$
as indicated by their subscripts.

We consider the following upper bound contraction probability, 
\begin{align}
 & B\cdot\max_{b=1,\cdots,B}\prod\left(\left.f_{b}\in\mathcal{F}:\left\Vert f_{0}-f_{b}\right\Vert _{n_{b}}>w_{b}(n)^{-1}\cdot M_{n_b}\varepsilon_{b,n}\right|\bm{y}_{b,n_{b}}\right)\nonumber \\
\geq & B\cdot\max_{b=1,\cdots,B}\prod\left(\left.f_{b}\in\mathcal{F}:\left\Vert f_{0}-f_{b}\right\Vert _{n}>w_{b}(n)^{-1}\cdot M_{n_b}\varepsilon_{b,n}\right|\bm{y}_{b,n_{b}}\right)\\
\geq & \sum_{b=1}^{B}\prod\left(\left.f_{b}\in\mathcal{F}:\left\Vert f_{0}-f_{b}\right\Vert _{n}>w_{b}(n)^{-1}\cdot M_{n}\varepsilon_{b,n}\right|\bm{y}_{b,n_{b}}\right)\nonumber \\
= & \sum_{b=1}^{B}\prod\left(\left.f_{b}\in\mathcal{F}:\left\Vert f_{0}-f_{b}\right\Vert _{n}>w_{b}(n)^{-1}\cdot M_{n}\varepsilon_{b,n}\right|\bm{y}_{n}\right)\label{eq:reason_1}
\end{align}
The second line follows from \eqref{eq:triangle_norm_emp}. The third
line follows from the definition of $M_n$ and taking maximum over $b=1,\cdots,B$.
The last line follows from the fact that the the $b$-th tree fits
with only the $b$-th shard of data, denoted as $\bm{y}_{b,n_{b}}$;
and $\cup_{b=1}^{B}\bm{y}_{b,n_{b}}=\bm{y}_{n}$. Then we use
the sub-additivity of the (posterior) probability  measure and yield from \eqref{eq:reason_1}
that 
\begin{align}
 & \sum_{b=1}^{B}\prod\left(\left.f_{b}\in\mathcal{F}:\left\Vert f_{0}-f_{b}\right\Vert _{n}>w_{b}(n)^{-1}\cdot M_{n}\varepsilon_{b,n}\right|\bm{y}_{n}\right)\nonumber \\
\geq & \prod\left(\left.\cup_{b=1}^{B}\left\{ f_{b}\in\mathcal{F}:\left\Vert f_{0}-f_{b}\right\Vert _{n}>w_{b}(n)^{-1}\cdot M_{n}\varepsilon_{b,n}\right\} \right|\bm{y}_{n}\right)\nonumber \\
\geq & \prod\left(\left.\sum_{b=1}^{B}w_{b}(n)\cdot\left\Vert f_{0}-f_{b}\right\Vert _{n}>M_{n}\sum_{b=1}^{B}\varepsilon_{b,n}\right|\bm{y}_{n}\right)\nonumber \\
\geq & \prod\left(\left.\sum_{b=1}^{B}w_{b}(n)\cdot\left\Vert f_{0}-f_{b}\right\Vert _{n}>M_{n}B\cdot\varepsilon_{n}\right|\bm{y}_{n}\right)\label{eq:reason_2}
\end{align}

The second line follows from sub-additivity of the (posterior) probability measure
$\prod$. The third line follows from the union bound, after moving
$w_{b}(n)$ to the left hand side and the definition that $0< w_{b}(n)\leq 1$. The last line follows from the
definition of $\varepsilon_{n}\coloneqq\min_{b=1,\cdots,B}\varepsilon_{b,n}$.
Using our assumption on the sum $\sum_{b=1}^{B}w_{b}(n)=1$ in \eqref{eq:reason_2},
we have 
\begin{align}
 & \prod\left(\left.\sum_{b=1}^{B}w_{b}(n)\cdot\left\Vert f_{0}-f_{b}\right\Vert _{n}>M_{n}B\cdot\varepsilon_{n}\right|\bm{y}_{n}\right)\nonumber \\
\geq & \prod\left(\left.\left\Vert \sum_{b=1}^{B}w_{b}(n)f_{0}-w_{b}(n)f_{b}\right\Vert _{n}>M_{n}B\cdot\varepsilon_{n}\right|\bm{y}_{n}\right)\nonumber \\
= & \prod\left(\left.\left\Vert f_{0}-\sum_{b=1}^{B}w_{b}(n)f_{b}\right\Vert _{n}>M_{n}B\cdot\varepsilon_{n}\right|\bm{y}_{n}\right)\label{eq:reason_3}
\end{align}
Now, we have proven \eqref{eq:main_bound}. Since $\epsilon$ is arbitrarily
chosen, this proves  the posterior concentration as well. 
\end{proof}

\section{\label{sec:Proof-of-Lemma}Proof of Lemma \ref{lem:(Computational-complexity-for}}
\begin{proof}
At each MCMC step, for each of the $m$ trees, a new tree structure
can be proposed given the previous tree structure of this specific
tree via one of the following 4 steps: (i) insert (i.e., grow) two
new child nodes at a leaf node (ii) delete (i.e., prune) two leaf
nodes and merge them into their parent node (iii) swap the splitting
criteria of two interior nodes (iv) change the splitting criteria
of a single interior node. We can consider (i) and (ii) as the regular
binary tree operations of insertion and deletion with complexity $\mathcal{O}(k_{i})$;
and (iii) and (iv) as search operation in a binary tree, also with
complexity $\mathcal{O}(k_{i})$.

When updating the structure of a single tree with at most $k_{i}$
nodes, it takes at most computational complexity $\mathcal{O}(k_{i})$.

Fixing the structure of the tree, when computing the acceptance probability,
the likelihood factors into 1-dimensional Gaussian terms (See, e.g.,
(23) in \citet{chipman_bart_2010} or (2) in \citet{tan2019bayesian})
and hence does not depend on the dimension of input but the bounded
by the total number of nodes in the tree. For numerators and denominators
in the likelihood ratio for accept-reject step, we have product in
the form of $$\prod_{j=1}^{m}\left(\prod_{\eta\in\mathcal{T}_{j}\text{ is a leaf}}\ell\left(\mu_{\eta}\left|\bm{X},\bm{y}\right.\right)\right)\cdot p\left(\mathcal{T}_{j}\right).$$
Each leaft node $\eta$ of the tree $\mathcal{T}_{j}$ contributes
one exponential likelihood entry, the computational complexity of
this double product is $\mathcal{O}\left(m\cdot N_{j,\text{leaf}}\cdot n\right)\lesssim\mathcal{O}\left(m\cdot n^2\right)$
where $N_{j,\text{leaf}}$ is the number of leaves in the tree $\mathcal{T}_{j}$
and cannot exceed the total number of samples $n$ for a complete
binary tree.

Therefore, for one step, updating a single tree takes computational
complexity $\mathcal{O}(m\cdot n^2+k_{i})$. Updating all $m$ trees
takes computational complexity $\mathcal{O}(m^{2}\cdot n^2+\sum_{i=1}^{m}k_{i})$.
\end{proof}

\section{SBT Implementation and Parameters}
Our software implementation (\url{https://bitbucket.org/mpratola/openbt/wiki/Home}) accepts the following API,
with support of MPI parallelism (By default, $\mathtt{shardpsplit}$=1).
\begin{align*}
\text{\texttt{ fit=openbt(x,y,pbd=c(1.0,0.0),ntree=m,}}\\ \texttt{ntreeh=1,numcut=100,tc=4,probchv=0.0,}\\
\text{\texttt{ shardepth=2,randshard=FALSE,}}\\
\text{\ensuremath{\mathtt{model="bart",modelname="branin")}}}
\end{align*}
There are several important parameters corresponding to our SBT algorithm
specified and explained below.

\textbf{ Model related parameters.} 
\begin{itemize}
\item $\mathtt{ntree}$: The number of trees (also denoted by $m$) used
in the (regular) BART model for each shard. 
\item $\mathtt{numcut}$: The number of maximal cuts in each of the trees
used in the BART model for each shard. 
\item $\mathtt{shardepth}$: The maximal depth of our sharding tree $\mathcal{T}_{u}$.
\item $\mathtt{randshard}$: This is a Boolean variable indicating whether
we split the original samples into shards for $\mathcal{T}_{x\mid\eta_{u}}$
deterministically or randomly. When this parameter is True, we will
take a number of samples from the original dataset according to its
given order for each shard; otherwise we will take a random sample
(without replacement for non-overlapping shards) from the original
dataset. 
\end{itemize}
\noindent %
\textbf{Other parameters.} 
\begin{itemize}
\item $\mathtt{tc}$: The number of MPI threads to use is specified as tc=4. 
\item $\mathtt{pbd}$: The parameter denotes the range of the birth and
death probability in the tree proposal. 
\item $\mathtt{ntreeh}=1$: The number of trees used in the heteregenous
BART model, the default value is 1, which is used in this paper. 
\item $\mathtt{probchv}$: The mean of the probability of a change variable
(of the splitting variable) move among all $\mathtt{ntree}$ trees
in BART, the default value is 0.1. 
\item $\mathtt{probchvh}$: The variance of the probability of a change
variable (of the splitting variable) move among all $\mathtt{ntree}$
trees in BART, the default value is 0.1. 
\item $\mathtt{model="bart"}$: The model $\mathcal{T}_{x\mid\eta_{u}}$
for each shard, currently we only support the (regular) BART developed
by \citet{chipman_bart_2010}. 
\end{itemize}
\label{sec:API_SBT}
\end{document}